\documentclass[runningheads]{llncs}
\usepackage[T1]{fontenc}
\usepackage{graphicx}
\usepackage{booktabs}
\usepackage{multirow}
\usepackage{array}
\usepackage{amsmath,amssymb,mathtools}
\usepackage{bm}
\usepackage{algorithm}
\usepackage[noend]{algpseudocode}
\usepackage{xspace}
\usepackage{float}
\usepackage{url}
\usepackage{tabularx}
\usepackage{placeins}
\usepackage[hidelinks,colorlinks=true,linkcolor=blue,citecolor=blue]{hyperref}
\usepackage[nameinlink,noabbrev]{cleveref}
\crefname{figure}{Fig.}{Figs.}
\crefname{equation}{}{}

\newcommand{\method}{AdaReP\xspace}
\newcommand{\MPC}{\mathsf{MPC}}
\newcommand{\OPT}{\mathsf{OPT}}
\newcommand{\cost}{\mathrm{cost}}
\newcommand{\norm}[1]{\left\lVert #1 \right\rVert}

\begin{document}

\title{AdaReP: \underline{Ada}ptive \underline{Re-P}lanning under Model Mismatch for Neural World-Model Predictive Control}
\titlerunning{AdaReP for Neural World-Model Predictive Control}

\author{
Yutian Cheng\inst{1,2}$^{*}$
\and
Xiaojian Ma\inst{2}$^{*}$
\and
Xianhao Wang\inst{2,3}
\and
Min Yang\inst{2,3}
\and \\
Rongpeng Su\inst{2,3}
\and
Hangxin Liu\inst{2}
\and
Xi Chen\inst{2}
\and
Shuai Li\inst{1}$^{+}$
\and
Qing Li\inst{2}$^{+}$
}
\authorrunning{Y. Cheng et al.}

\institute{
Shanghai Jiao Tong University
\and
Beijing Institute for General Artificial Intelligence (BIGAI)
\and
University of Science and Technology of China
}

\begingroup
\renewcommand{\thefootnote}{}
\footnotetext{$^{*}$ These authors contributed equally.}
\footnotetext{$^{+}$ Corresponding authors.}
\footnotetext{Accepted at ICANN 2026.
This arXiv version contains supplementary materials and appendices
that are omitted from the conference version due to space limitations.}
\endgroup

\maketitle

% \maketitle

\begin{abstract}
Neural world models coupled with model predictive control (MPC) replan at every environment step to bound accumulated prediction error, but this incurs substantial computational overhead. Reusing a cached plan reduces this overhead, yet its effectiveness depends on how prediction mismatch propagates through the local dynamics. We analyze this trade-off with a perturbation-based dynamic-regret framework and show that stale-plan penalties scale with the reuse tolerance, the accumulated mismatch since the last replanning step, and the local dynamics sensitivity. Based on this structure, we propose \method, a training-free wrapper that adapts the replanning tolerance online using the current deviation from the cached rollout and a local sensitivity estimate, without modifying the learned world model or planner. Across image-space planning, latent-space control, and real-world robotic manipulation, \method substantially reduces planner-side computation while maintaining comparable task performance, including over 80\% fewer queries on a 50-trial physical robot study.
\keywords{Adaptive replanning \and Neural world models \and Model predictive control \and Neurorobotics \and Online control.}
\end{abstract}

\section{Introduction}
Neural world models provide a predictive interface for control in robotics. Recent systems plan in image space \cite{tian2023vp2}, latent space \cite{hafner2019planet,hafner2019dream,hafner2023dreamerv3,hansen2022temporal,hansentd}, or other learned state spaces \cite{wu2024ivideogpt,zhao2024vlmpc} using large neural networks as predictors. These systems typically replan at every environment step to bound accumulated prediction error, but each replanning call requires many queries to the large world model, resulting in substantial computational overhead, which is further exacerbated by modern large-scale, large-parameter world models. Intuitively, when execution closely tracks the predicted rollout, reusing the cached plan can save computation without sacrificing performance, but continued reuse risks degradation when execution drifts \cite{tabuada2007event} or local dynamics are highly sensitive \cite{heemels2012introduction}.

We therefore study \emph{replanning cadence} as a \textbf{control variable} for neural world-model predictive control. Specifically, we do not modify the learned world model, retrain the policy, or redesign the MPC solver. Instead, given an existing predictive controller, we address the question: at what point during execution should the current plan be refreshed? Using perturbation-style dynamic-regret arguments for predictive control under model mismatch \cite{lin2021perturbation,lin2022bounded}, we show that stale-plan penalties grow with three quantities: the reuse tolerance, the accumulated prediction mismatch since the last replanning step, and the local dynamics sensitivity along the executed trajectory. This structure explains why a single fixed reuse schedule rarely transfers across tasks or even across phases of the same task.

Motivated by this analysis, we develop \method, a training-free replanning layer for neural world-model MPC. \method monitors the current deviation from the cached rollout together with a simple local-sensitivity estimate, converts these signals into a time-varying reuse tolerance, and replans only when the cached plan has become unreliable. The same mechanism wraps image-space planners, latent-space planners, or real-world state-space controllers without modification.

We evaluate \method on VP2/RoboDesk \cite{tian2023vp2}, TD-MPC2 on the 30-task DeepMind Control suite \cite{hansentd,tunyasuvunakool2020}, and real-world Franka manipulation. Across all three settings, \method reduces planner-side model queries while maintaining comparable task performance. On Franka, AdaReP achieves aggregate success $36/50$ versus $34/50$ for step-wise replanning while cutting planner-side queries by more than 80\%. The experiments further reveal why fixed reuse schedules are hard to tune globally: the preferred replanning cadence varies across backbones, tasks, and even across different stages of the same physical trajectory.

\begin{figure}[t]
    \centering
    \includegraphics[width=0.98\linewidth]{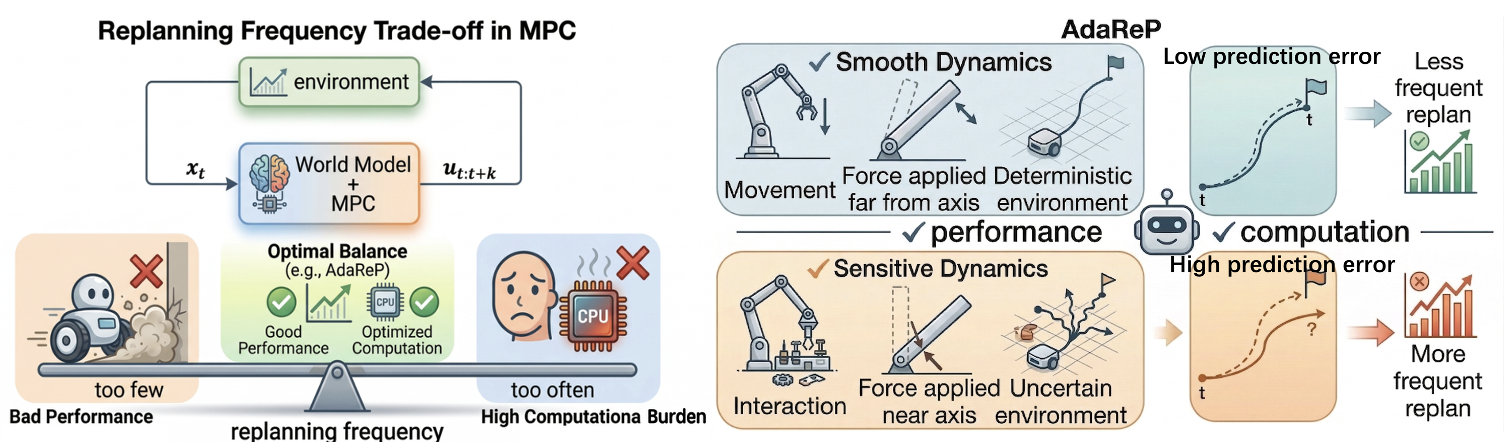}
    \caption{AdaReP addresses the computational efficiency–control performance trade‑off in MPC. Traditional solvers replan frequently to bound cumulative prediction error, incurring substantial computational overhead. AdaReP adjusts replanning frequency on‑the‑fly using online estimates of prediction error and local dynamics sensitivity, thereby reducing computation while preserving control performance.}
    \label{fig:overview}
\end{figure}

Our contributions are as follows:
\begin{itemize}
    \item We derive \textbf{dynamic-regret bounds} for fixed-step, fixed-threshold, and adaptive replanning under model mismatch, identifying the mismatch and sensitivity terms that motivate adaptive thresholding.
    \item We propose \method, a \textbf{training-free} adaptive replanning rule that leverages online deviation and a local sensitivity proxy to determine when a cached plan should be refreshed.
    \item We evaluate \method on image-space, latent-space, and real-robot control benchmarks, demonstrating substantial reductions in world-model queries while requiring only \textbf{backbone-level tuning}.
\end{itemize}

\section{Related Work}
\paragraph{Neural world models and predictive control.}
Neural world-model methods cover latent-dynamics planning in PlaNet, Dreamer, and DreamerV3 \cite{hafner2019planet,hafner2019dream,hafner2023dreamerv3}, latent MPC in TD-MPC and TD-MPC2 \cite{hansen2022temporal,hansentd}, recurrent or object-centric video prediction for manipulation \cite{villegas2019high,minderer2019unsupervised,babaeizadeh2021fitvid,tian2023vp2,Tian_2023_CVPR}, and recent transformer, diffusion, and multimodal predictors \cite{wu2024ivideogpt,zhao2024vlmpc,dingUnderstandingWorldPredicting2024}. Most of this literature improves the predictive backbone, representation, or planning objective. Our focus is complementary: we keep the learned model and planner fixed and instead study the execution-time policy that decides \emph{when} replanning is necessary.

\paragraph{Online control, regret, and computation-aware updates.}
A second line of work connects predictive control to online learning and analyzes regret under forecasts or model mismatch \cite{wagener2019online,li2019online,yu2020power,zhang2021regret,lin2021perturbation,lin2022bounded,muthirayan2022online}. These works study the performance of online or predictive control policies, often in linear or structured time-varying settings. Classical event-triggered and self-triggered control \cite{tabuada2007event,heemels2012introduction}, as well as warm-start or fast MPC variants \cite{jackson2021altro,marcucci2020warm,ma2020event}, also reduce update frequency by linking refreshes to state-dependent conditions. We borrow the structural insight from these literatures, namely that stale updates should be penalized according to local mismatch and sensitivity, but target a different regime: learned neural predictive control, where the trigger is built from discrepancy signals in state or feature space rather than from a known control-theoretic model. This separation lets the same adaptive rule transfer across image-space, latent-space, and real-robot control with backbone-level tuning instead of task-specific redesign.

\section{Problem Setup}
We consider a finite-horizon control problem over horizon $T$. The planning state and action are denoted $x_t\in\mathcal X$ and $u_t\in\mathcal U$, respectively. The environment evolves according to the true dynamics
\begin{equation}
 x_{t+1} = g_t(x_t,u_t),
\label{eq:true_dynamics_main}
\end{equation}
and incurs cumulative cost
\begin{equation}
\cost(x_{0:T},u_{0:T-1}) \coloneqq \sum_{t=0}^{T-1} f_t(x_t,u_t) + F_T(x_T).
\label{eq:episode_cost_main}
\end{equation}

At time $t$, the controller plans with a learned predictive model $\hat g_t$ and an MPC solver over a horizon $k$. We write the resulting planning map as
\[
(y_{t:t+k}, v_{t:t+k-1}) = \psi_t^k(x_t),
\]
where $y_{t:t+k}$ is the model-predicted rollout and $v_{t:t+k-1}$ is the corresponding action sequence. In practice, $\psi_t^k$ may be instantiated by a sampling-based MPC routine such as CEM \cite{de2005tutorial} or MPPI \cite{williamsInformationTheoreticMPC2017} on top of an image-space, latent-space, or state-space predictor.

% We write $\MPC_k $ for the standard receding-horizon controller that replans at every step. %To discuss reduced-frequency replanning, let $p(t)$ denote the most recent replanning time before step $t$. A plan computed at time $p(t)$ provides a cached predicted state $\hat y_{t\mid p(t)}$ for the current step together with the remaining action suffix; for brevity we write $\hat y_t \equiv \hat y_{t\mid p(t)}$.
% We compare against two reduced-frequency baselines. The fixed-step controller $\MPC_k^m$ replans every $m$ steps and executes the cached suffix in between refreshes. The fixed-threshold controller $\MPC_{k,\epsilon}$ replans whenever the discrepancy between execution and the cached rollout exceeds a constant tolerance $\epsilon$.

The controller's planning and execution operate on a representation $z_t \coloneqq \phi(x_t)$, which is the monitored quantity for the trigger. In image-space controllers, $z_t$ is a frozen image-feature embedding \cite{tian2023vp2}; in latent-space controllers, the planner latent serves both roles \cite{hansen2022temporal,hansentd}; and in physical-state controllers, $z_t$ coincides with the raw state $x_t$ (i.e., $\phi$ is the identity). We define the corresponding predicted representation as $\hat z_t \coloneqq \phi(y_t)$. The monitored discrepancy is
\begin{equation}
\label{eq:trigger_metric}
d_t \coloneqq \norm{z_t-\hat z_t}_2 .
\end{equation}
% The controller $\MPC_{k,\epsilon}$ reuses the cached plan while $d_t\leq\epsilon$ and the cached suffix has not reached its horizon end. \method uses the same trigger with a time-varying tolerance $\epsilon_t\in(0,\epsilon_0]$. After executing $u_t$ and observing $x_{t+1}$, it computes $d_{t+1}$ and decides whether to reuse the plan for step $t+1$ or to refresh it before that step.

We write $\MPC_k^1 $ for the standard receding-horizon controller that replans at every step. %To discuss reduced-frequency replanning, let $p(t)$ denote the most recent replanning time before step $t$. A plan computed at time $p(t)$ provides a cached predicted state $\hat y_{t\mid p(t)}$ for the current step together with the remaining action suffix; for brevity we write $\hat y_t \equiv \hat y_{t\mid p(t)}$.
We compare against two reduced-frequency baselines. The fixed-step controller $\MPC_k^m$ replans every $m$ steps and executes the cached suffix in between refreshes. The fixed-threshold controller $\MPC_{k,\epsilon}$ replans whenever the discrepancy between execution and the cached rollout exceeds a constant tolerance $\epsilon$.

Our analysis assumes that the true dynamics are locally Lipschitz around the executed and comparator trajectories: for each $t$ there exists $L_t>0$ such that
\begin{equation}
\norm{g_t(x,u)-g_t(x',u')}_2 \leq L_t\big(\norm{x-x'}_2 + \norm{u-u'}_2\big).
\label{eq:local_lipschitz_main}
\end{equation}
The coefficient $L_t$ is the local dynamics sensitivity at time $t$: it quantifies how strongly one-step state or action perturbations are amplified by the true dynamics. We also write
\(
L \coloneqq \max_{0\leq t<T} L_t\) and
\(
L^{(m)} \coloneqq \max_{0\leq t\leq T-k}\max_{0\leq i\leq m-1}\prod_{s=t}^{t+i} L_s,
\)
so that $L$ captures worst-case single-step sensitivity and $L^{(m)}$ captures the worst cumulative sensitivity over an $m$-step reuse block. In particular, $L^{(m)}\leq L^m$.

Following the online-control viewpoint in predictive control \cite{wagener2019online,li2019online,lin2021perturbation,lin2022bounded}, we measure control quality by dynamic regret,
\[
\mathrm{Regret}(\mathrm{ALG}) \coloneqq \cost(\mathrm{ALG}) - \cost(\OPT),
\]
where $\OPT$ is the clairvoyant optimum that minimizes \eqref{eq:episode_cost_main} under the true dynamics \eqref{eq:true_dynamics_main}. Computation is measured by the number of function evaluations (NFE), namely the total number of world-model queries issued by the planner during one episode.

\section{Adaptive Replanning under Model Mismatch}

This section presents the key analytical results and insights behind our trigger design. Due to page limits, we focus on the distilled conclusions that are essential for the algorithmic interpretation; complete derivations, assumptions, and proofs are deferred to the supplementary material.

\subsection{Regret structure under fixed-step and fixed-threshold reuse}
% Two quantities drive stale-plan error: the accumulated mismatch since the last replanning step, and the local dynamics' amplification of that mismatch, captured by the sensitivity factors $L_t$ and $L^{(m)}$ from \cref{eq:local_lipschitz_main}.

To understand when replanning is needed, we first identify the core quantities that govern the error incurred by reusing a cached plan under model mismatch.
Our derivation adapts the perturbation-based dynamic-regret analysis of \cite{lin2021perturbation,lin2022bounded} to cached-plan reuse. We record only the resulting dependence on reuse length, tolerance, mismatch, and local sensitivity.%The proof first bounds the deviation induced by reusing a cached plan and then feeds that deviation into the standard perturbation-to-regret reduction used in prior predictive-control work. We record only the resulting dependence on reuse length, tolerance, mismatch, and local sensitivity.

\begin{theorem}[Fixed-step and fixed-threshold regret bounds]
\label{thm:condensed_bounds}
Under the local perturbation assumptions of \cite{lin2021perturbation,lin2022bounded} and the small-error condition that keeps the trajectory inside the local tube around the clairvoyant optimum,
\begin{align}
\mathrm{Regret}(\MPC_k^m)
&= O\!\left(\sqrt{m(L^{(m)})^2\cost(\OPT)\,E} + m(L^{(m)})^2E\right),
\label{eq:fixed_m_regret_main}
\\
\mathrm{Regret}(\MPC_{k,\epsilon})
&= O\!\left(\sqrt{L^2\cost(\OPT)\,(E+\epsilon E+\epsilon^2T)}\right. \nonumber\\
&\qquad\left.+\,L^2(E+\epsilon E+\epsilon^2T)\right),
\label{eq:fixed_eps_regret_main}
\end{align}
where   $E$ denotes the total cumulative prediction error over all reused segments in an episode.
\end{theorem}

For fixed-step reuse, the extra term $m(L^{(m)})^2E$ exposes the {core difficulty}: mismatch accumulates throughout the reuse block, and that accumulated mismatch is then amplified by the multi-step sensitivity factor $L^{(m)}$. Since $L^{(m)}\leq L^m$, the penalty can grow \emph{exponentially} with the reuse length. %in locally sensitive regimes.

For fixed-threshold reuse, the additional dependence is
\(
\epsilon L^2E + \epsilon^2L^2T.
\)
The mixed term couples tolerated deviation with the cumulative mismatch budget accumulated under stale reuse, while the quadratic term penalizes maintaining a nonzero tolerance throughout the episode. A single fixed threshold is therefore \textbf{brittle} whenever mismatch and sensitivity vary over time.

\subsection{Adaptive thresholded replanning}
\begin{algorithm}[t]
    \caption{\method: deployment-time adaptive replanning (practical version)}
    \label{alg:adarep}
    \begin{algorithmic}[1]
    \Require Base threshold $\epsilon_0$, weights $\alpha_\delta, \alpha_L$, smoothing window $W$, stabilizer $\gamma>0$, and base planner
    \State Replan at $t=0$ to obtain an active trajectory $( y_{0:k}, v_{0:k-1})$
    \For{$t=0,1,\dots,T-1$}
        \State Execute the current control $u_t$ and observe $x_{t+1}$
        \State Compute deviation $d_{t+1}$ and sensitivity estimate $\widehat{L}_t$ via \cref{eq:trigger_metric,eq:local_sensitivity_surrogate}
        \State Update sliding-window averages $\bar d_{t+1}$ and $\bar L_t$ over the last $W$ values
        \State Set $\epsilon_{t+1}$ via \cref{eq:adaptive_threshold}
        \If{$d_{t+1} > \epsilon_{t+1}$ \textbf{or} $t+1$ reaches the end of the cached plan}
            \State Replan and refresh the active trajectory for step $t+1$
        \Else
            \State Reuse the remaining plan suffix
        \EndIf
    \EndFor
    \end{algorithmic}
    \end{algorithm}
The fixed-threshold bound points to the core design principle: the reuse tolerance should contract when stale mismatch or local sensitivity grows. Inspecting the proof more closely, the threshold-dependent regret contribution appears through per-step terms of the form $L_t^2\epsilon_t^2$ and $L_t^2\epsilon_t s_t$, where $s_t$ is a stale-mismatch statistic capturing multi-step prediction errors along the cached rollout. This term is, however, not available at deployment time since it requires access to the full cached trajectory.

Motivated by this structure, we seek a deployable algorithm that adapts the reuse tolerance online as shown in  \cref{alg:adarep}. 
%The challenge is to replace the inaccessible $s_t$ with signals that can be measured during execution. Observing that $s_t$ quantifies how far the execution has drifted from the cached rollout, we can directly monitor this divergence. Meanwhile, the sensitivity term $L_t$ can be approximated by a lightweight finite-difference estimate. 
\method tracks two execution-time quantities: the current deviation from the cached rollout,
\[
 d_{t+1}=\norm{z_{t+1}-\hat z_{t+1}}_2,
\]
and a finite-difference sensitivity proxy computed in the same monitored space,
\begin{equation}
\widehat{L}_t \coloneqq \frac{\norm{z_{t+1}-z_t}_2}{\norm{u_t}_2+\gamma},
\label{eq:local_sensitivity_surrogate}
\end{equation}
with a small stabilizer $\gamma>0$ to avoid spurious inflation when $\norm{u_t}_2$ is close to zero and $\widehat L_t$ indicates how strongly that staleness is likely to be amplified by the local dynamics. 

%To motivate the update rule, 
Consider the  threshold
\(
\tilde\epsilon_t = \epsilon_0 \exp(-\alpha_\delta d_t)\exp(-\alpha_L \widehat L_t).
\)
%Under the idealized monitor-dominance conditions used only to motivate the threshold form --- namely, that the monitored discrepancy upper-bounds stale mismatch and $\widehat L_t$ lower-bounds local sensitivity up to fixed constants --- 
the exponential form directly controls the two threshold-dependent penalties above:
\begin{equation}
L_t^2\tilde\epsilon_t^2 \lesssim \frac{\epsilon_0^2}{\alpha_L^2},
\qquad
L_t^2\tilde\epsilon_t s_t \lesssim \frac{\epsilon_0}{\alpha_L^2\alpha_\delta}.
\label{eq:adaptive_penalty_control_main}
\end{equation}
These inequalities follow from the elementary bounds $xe^{-ax}\leq (ea)^{-1}$ and $x^2e^{-ax}\leq 4(e^2a^2)^{-1}$. Exponential decay is therefore not an arbitrary modeling choice: it is the \emph{simplest} way to convert the sensitivity-amplified terms $L_t^2\epsilon_t^2$ and $L_t^2\epsilon_t s_t$ into uniform contributions controlled by $(\epsilon_0,\alpha_\delta,\alpha_L)$.

In deployment we apply a short smoothing window $W$ for numerical stability and set
\begin{equation}
\label{eq:adaptive_threshold}
\epsilon_{t+1} = \epsilon_0 \exp(-\alpha_\delta \bar d_{t+1})\exp(-\alpha_L \bar L_t).
\end{equation}
%Here $\epsilon_{t+1}$ is the threshold used to decide whether the cached plan should be refreshed before step $t+1$. 
This keeps $\epsilon_{t+1}\in(0,\epsilon_0]$ and reduces the allowed reuse whenever either the current deviation or the local sensitivity estimate increases.

% The practical version uses smoothed statistics for numerical stability; the full threshold-update subroutine is listed in \cref{alg:supp_threshold_update} of the supplementary material.

% \subsection{Adaptive regret bound}
Finally we give theoretical gaurantee for our proposed algorithm. Applying the same perturbation-to-regret reduction with the idealized threshold \eqref{eq:adaptive_threshold} yields an adaptive bound with explicit dependence on $(\epsilon_0,\alpha_\delta,\alpha_L)$.

\begin{theorem}%[Adaptive regret bound for exponential thresholding]
\label{thm:adaptive_regret_summary}
Under the local perturbation assumptions of \cite{lin2021perturbation,lin2022bounded}, \method with threshold \eqref{eq:adaptive_threshold} satisfies
\begin{align}
\mathrm{Regret}(\method)
&= O\!\left(\sqrt{\cost(\OPT)\left(L^2E + \frac{\epsilon_0}{\alpha_L^2}\left(\epsilon_0 + \frac{1}{\alpha_\delta}\right)T\right)}\right.\nonumber\\
&\qquad\left.+\,L^2E + \frac{\epsilon_0}{\alpha_L^2}\left(\epsilon_0 + \frac{1}{\alpha_\delta}\right)T\right).
\label{eq:adaptive_regret_main}
\end{align}
\end{theorem}

% Relative to the fixed-threshold bound, the explicit dependence $\epsilon L^2E+\epsilon^2L^2T$ is replaced by
% \[
% \frac{\epsilon_0}{\alpha_L^2}\left(\epsilon_0+\frac{1}{\alpha_\delta}\right)T.
% \]
Larger $\alpha_L$ more strongly suppresses the quadratic sensitivity term, while larger $\alpha_\delta$ more strongly suppresses the mixed mismatch term.% This is the main theoretical motivation for the exponential threshold used by \method.

\section{Experimental Evaluation}
We evaluate \method across three experimental domains of increasing realism to answer: \emph{Can \method reduce computational cost while maintaining task performance across diverse world models and physical platforms?}

\begin{enumerate}
    \item \textbf{Simulation with Image-space World Models.} We first evaluate on the VP2 benchmark \cite{tian2023vp2} using the RoboDesk manipulation environment with two predictive models: SVG \cite{villegas2019high} and Struct-VRNN \cite{minderer2019unsupervised}. This covers 7 manipulation tasks including button pushing, drawer opening, and block manipulation.

    \item \textbf{Simulation with Latent-space World Models.} We further test on the DeepMind Control Suite \cite{tunyasuvunakool2020} using TD-MPC2 \cite{hansentd}, a SOTA model-based RL agent that plans in learned latent spaces across 30 continuous control tasks.

    \item \textbf{Physical Real-world Deployment.} Finally, we validate on a real Franka Emika Panda arm performing door opening and T-block rearrangement using a state-based world model, to assess whether simulation gains transfer to physical robotics.
\end{enumerate}

\subsection{Main Results}
% \textbf{Across all three settings}, \method substantially reduces planner-side NFE while maintaining or exceeding the task performance of step-wise replanning.

\begin{figure}[h]
    \centering
    \includegraphics[width=0.95\linewidth]{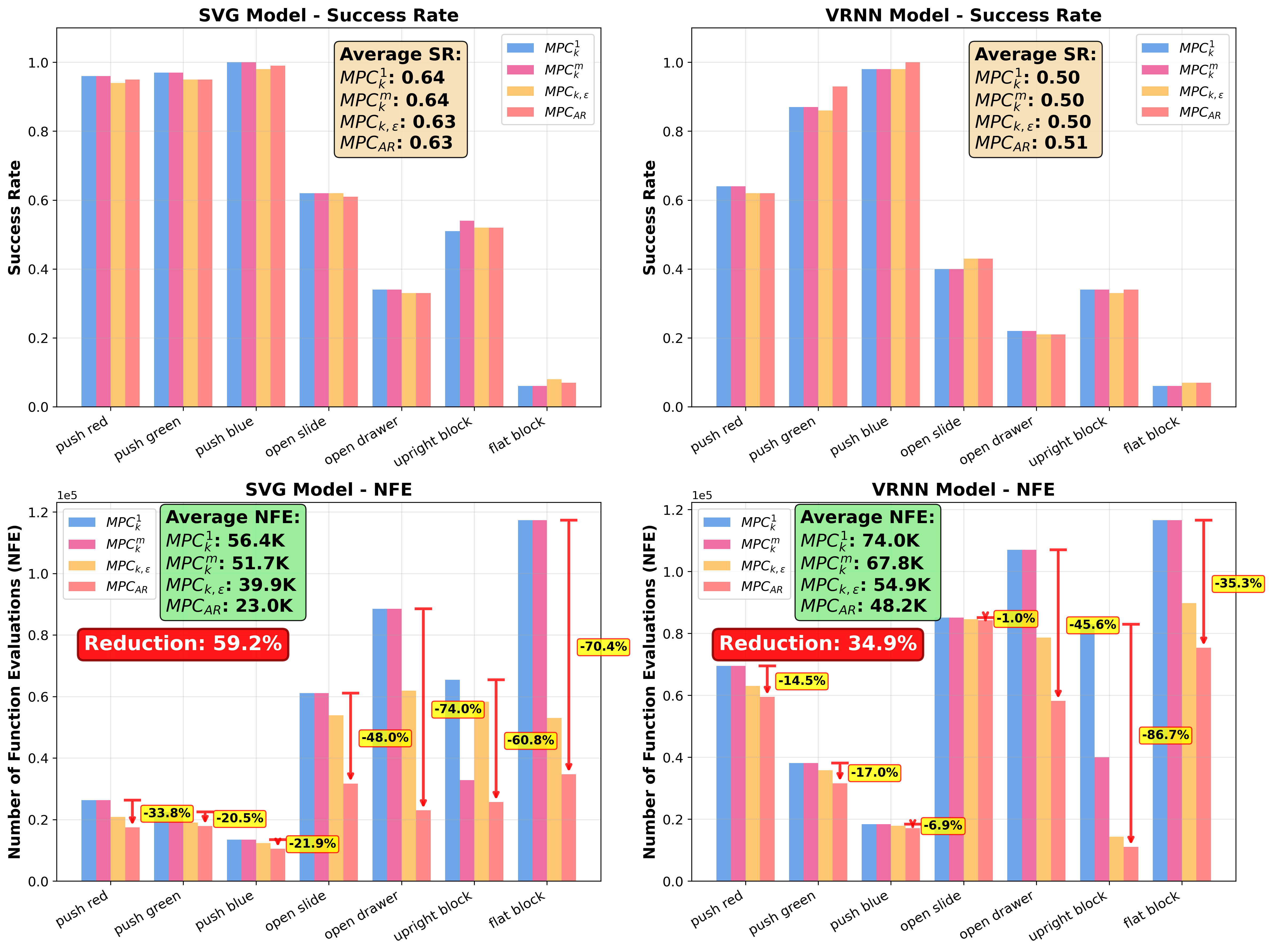}
    \caption{VP2 results on RoboDesk with SVG and Struct-VRNN under the matched-performance protocol. \method achieves the lowest NFE on both backbones, reducing NFE by \textbf{59.2\%} (SVG) and \textbf{34.9\%} (VRNN) with wall-clock reductions of \textbf{57.8\%} and \textbf{31.5\%}.}
    \label{fig:vp2_results_highres}
\end{figure}

\begin{figure}[h]
    \centering
    \includegraphics[width=0.95\linewidth]{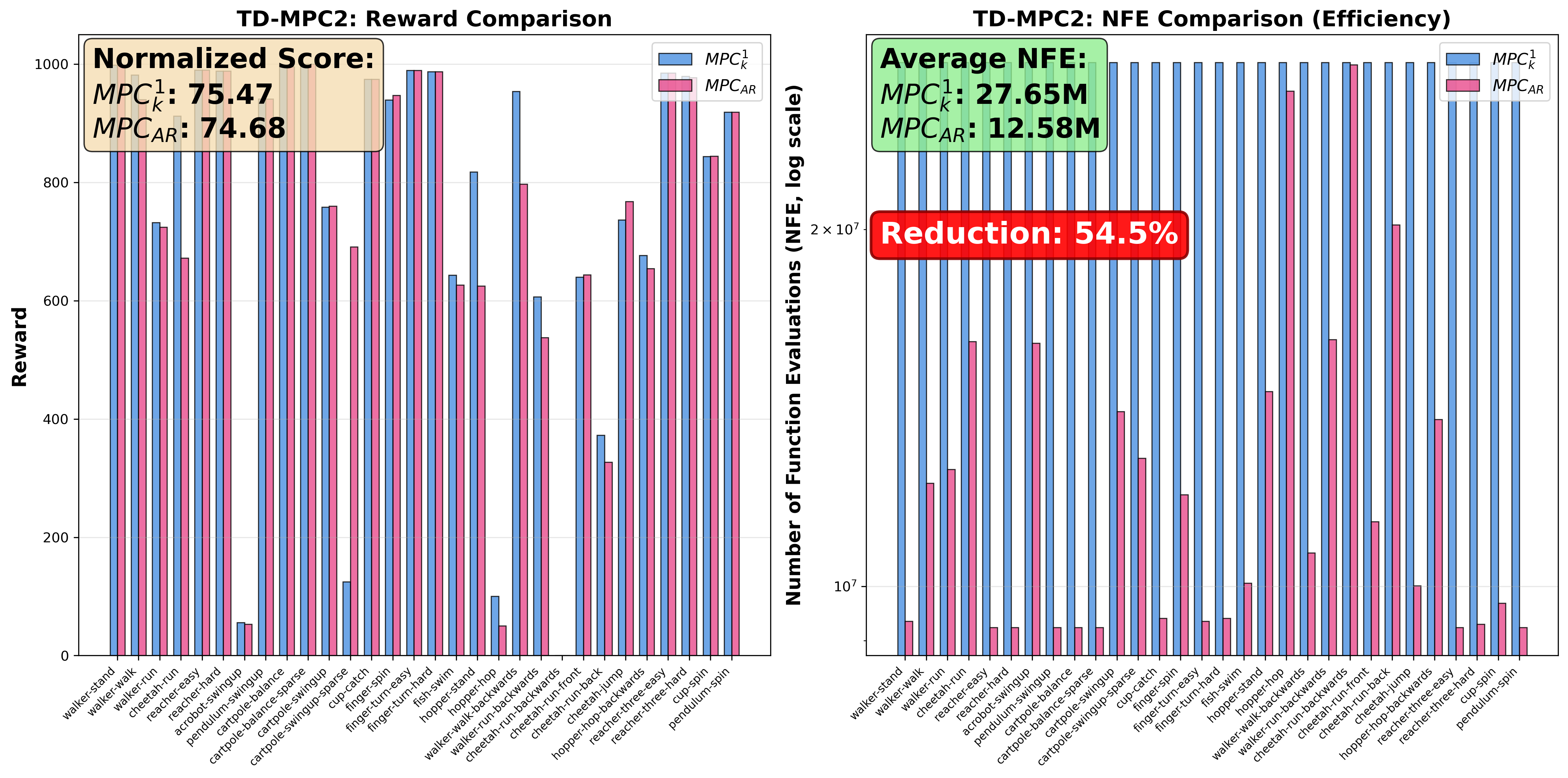}
    \caption{DeepMind Control results with TD-MPC2. \method reduces average NFE by \textbf{54.5\%} and wall-clock time by \textbf{50.3\%} while closely matching the mean normalized score of step-wise replanning across 30 tasks.}
    \label{fig:tdmpc2_results_highres}
\end{figure}

\begin{figure}[h]
    \centering
    \includegraphics[width=0.95\linewidth]{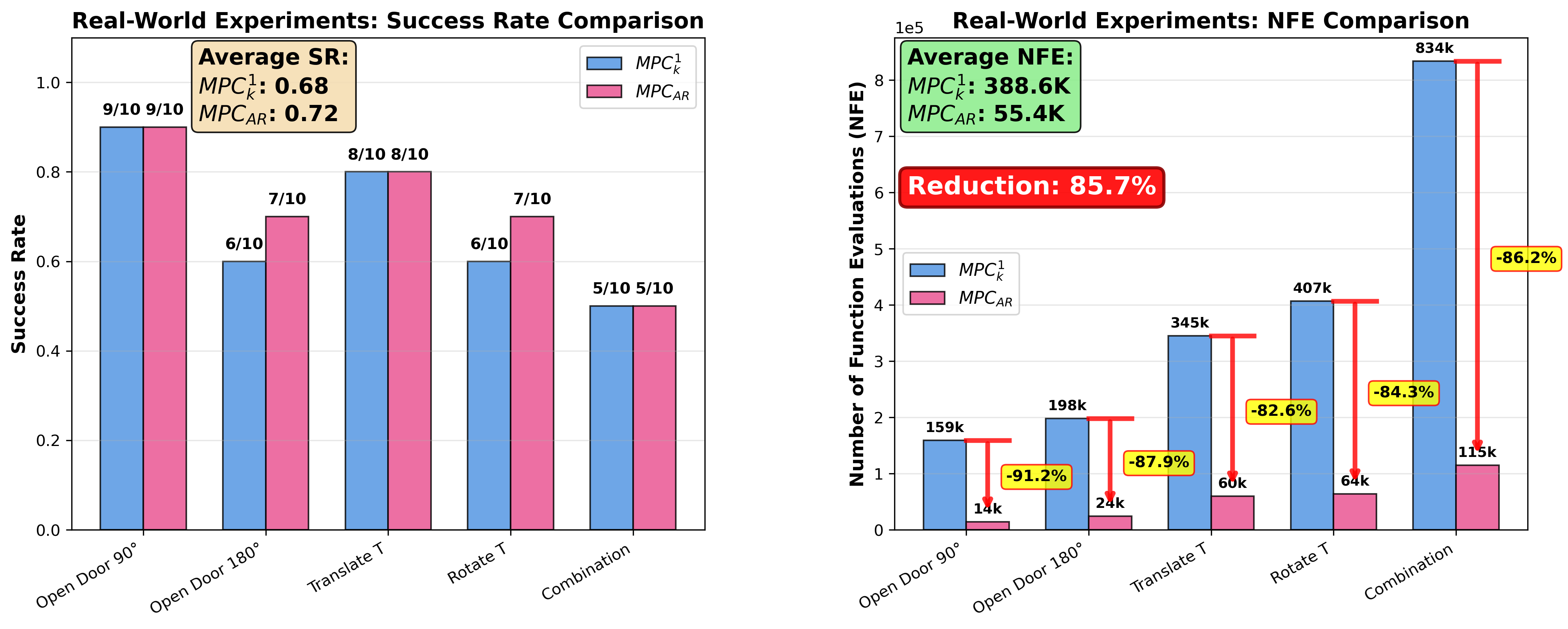}
    \caption{Real-world Franka results. \method achieves comparable success  while cutting planner-side NFE by over \textbf{80\%} across both articulation and rearrangement tasks.}
    \label{fig:franka_results_highres}
\end{figure}

On VP2 (\cref{fig:vp2_results_highres}),  we tune hyperparameters for all methods to ensure success rates drop no more than \textbf{0.02} compared to the standard $\MPC_k^1$ baseline. Under the matched-performance protocol, \method achieves the lowest NFE on both, reducing NFE by \textbf{59.2\%} (SVG) and \textbf{34.9\%} (Struct-VRNN) with wall-clock reductions of \textbf{57.8\%} and \textbf{31.5\%} respectively.

On DeepMind Control with TD-MPC2 (\cref{fig:tdmpc2_results_highres}), \method reduces average NFE by \textbf{54.5\%} and wall-clock time by \textbf{50.3\%} while matching the mean normalized score of step-wise replanning across 30 tasks.

On physical hardware (\cref{fig:franka_results_highres}), \method cuts planner-side queries by over \textbf{80\%} while achieving comparable success (36/50 vs. 34/50). State-based world models are more explicit than vision-based ones, which explains the stronger NFE gains on the real robot than in simulation.
\subsection{Analysis and Discussion}

\paragraph{\method adapts to prediction accuracy.}
To test how prediction quality affects \method, we introduce controlled Gaussian noise on state components (robot position/velocity, object position/velocity, end-effector position) in a disturbed simulator. As shown in \cref{fig:diagnostics_main}, the prediction discrepancy signal $d_t$ grows with noise, the sensitivity estimate $\widehat L_t$ adapts accordingly, and the algorithm naturally triggers more frequent replanning. \textbf{Crucially}, \method never degrades below step-wise replanning success rates; it gracefully falls back toward step-wise replanning in the limit of large mismatch.

\begin{figure}[h]
    \centering
    \begin{minipage}[b]{0.98\linewidth}
        \centering
        \includegraphics[width=\linewidth]{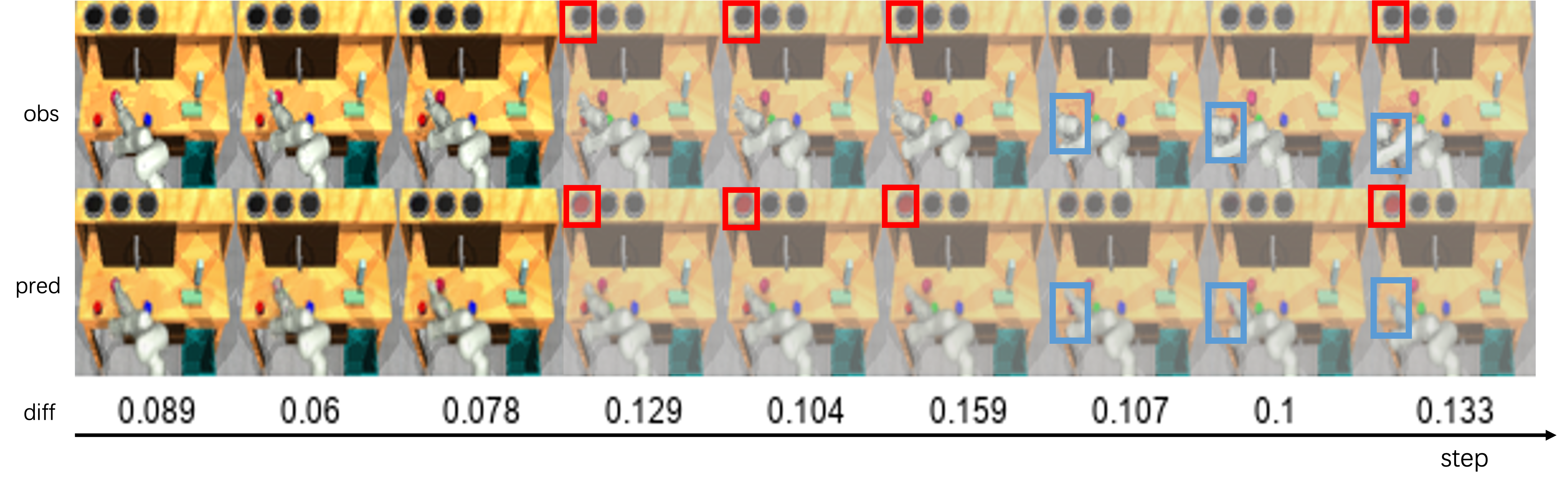}
    \end{minipage}\\[0.35em]
    \begin{minipage}[b]{0.98\linewidth}
        \centering
        \includegraphics[width=\linewidth]{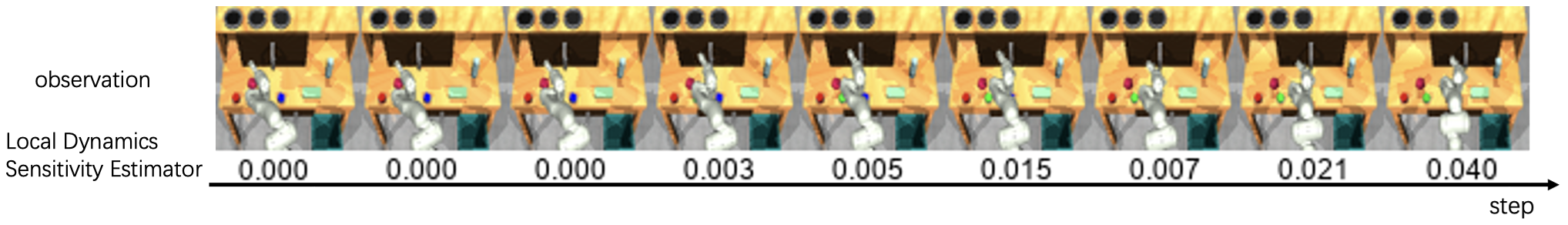}
    \end{minipage}
    \caption{Online monitors used by \method. Top: prediction discrepancy $d_t$ between execution and the cached rollout; highlighted regions mark when the cached plan has become stale. Bottom: local sensitivity estimate $\widehat L_t$ along a door-opening trajectory; sensitivity rises near the hinge, where stale plans are most dangerous.}
    \label{fig:diagnostics_main}
\end{figure}

\paragraph{\method adapts to intra-trajectory dynamics sensitivity.}
The effective sensitivity of the dynamics varies within a single episode. In the door-opening task, dynamics are highly sensitive near the hinge axis---small end-effector motions induce large angular changes---while interactions far from the hinge are smoother. \method adapts by shortening reuse near the hinge and extending it in smoother phases (\cref{fig:tradeoff_sweeps}), a capability unattainable with a fixed cadence. This behavior aligns precisely with the regime dependence identified in \cref{eq:fixed_eps_regret_main}.

\begin{figure}[h]
    \centering
    \includegraphics[width=0.9\linewidth]{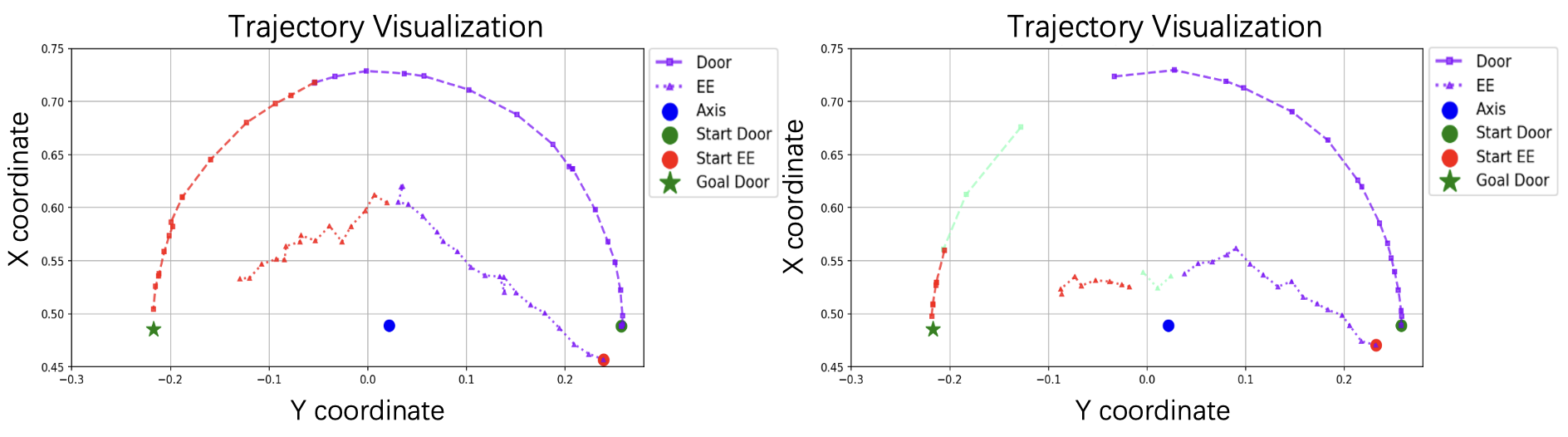}
    \caption{Dynamics sensitivity varies within a single door-opening episode. Far from the hinge (left), motion is smooth and reuse can be extended. Near the hinge (right), small end-effector motions cause large state changes, and reuse must be shortened.}
    \label{fig:tradeoff_sweeps}
\end{figure}

\paragraph{\method maintains performance in worst-case conditions.}
Real-world applications often present challenges such as large, unexpected prediction errors or highly sensitive dynamics. As shown in \cref{fig:diagnostics_main} (top), the prediction discrepancy signal $d_t$ grows with state disturbance, prompting more frequent replanning; and as shown in \cref{fig:robustness_analysis}, \method's success rate is preserved across all disturbance levels. In such adverse conditions, \method automatically reduces its reuse tolerance, falling back toward step-wise replanning. This graceful degradation ensures that performance never drops below the step-wise replanning baseline, at the cost of reduced NFE savings in demanding regimes.

\paragraph{Robust visual features remain a challenge.}
Even with a perfectly accurate underlying state predictor, corrupting images with Gaussian noise or blur degrades the DINO-based feature signal, causing the discrepancy monitor to trigger unnecessary replanning and reducing NFE savings. The robustness analysis (\cref{fig:robustness_analysis}) confirms this sensitivity: under visual corruption, the NFE advantage of \method narrows substantially. Developing more robust visual feature extractors is a valuable direction for extending \method to noisy vision-based settings.

\begin{figure}[t]
    \centering
    \includegraphics[width=0.9\linewidth]{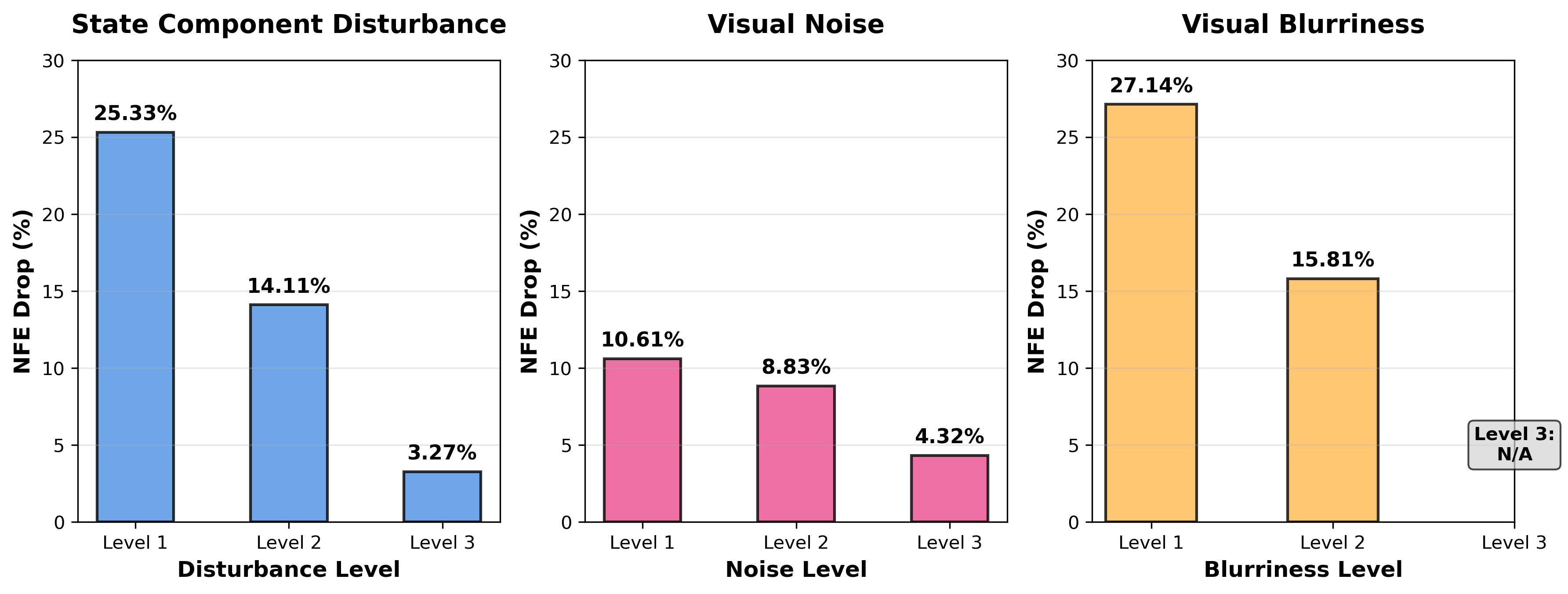}
    \caption{Impact of state disturbance, visual noise, and visual blurriness on the NFE reduction of \method relative to step-wise replanning. Success rates are maintained across all conditions. Visual corruption notably diminishes the NFE advantage, highlighting the importance of robust feature extractors.}
    \label{fig:robustness_analysis}
\end{figure}
\begin{figure}[h]
    \centering
    \begin{minipage}[c]{0.48\linewidth}
        \centering
        \includegraphics[width=\linewidth]{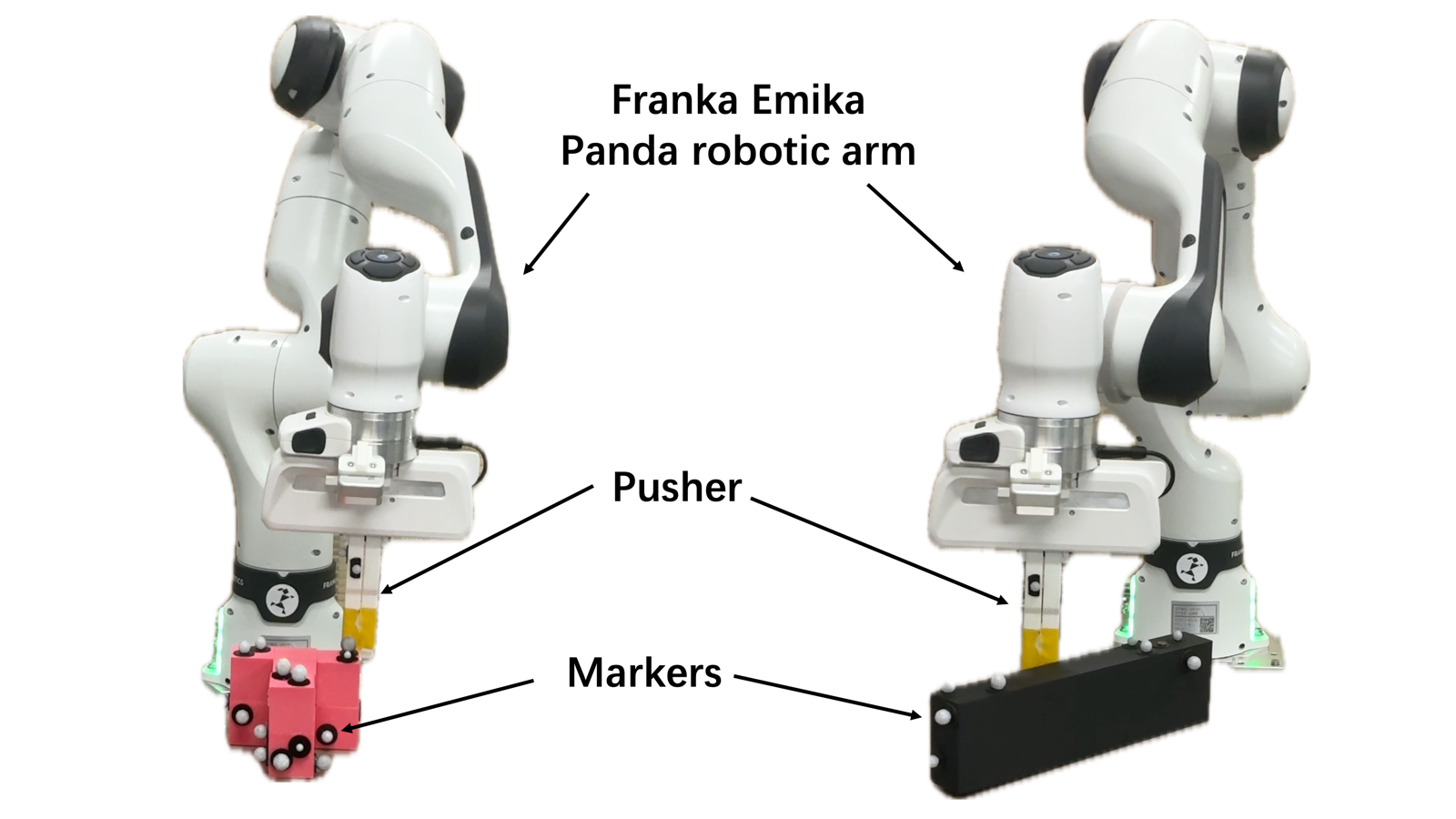}
        \caption{Physical evaluation platform. The Franka Emika Panda is used for articulation tasks (door opening) and long-horizon rearrangement tasks (T-block manipulation).}
        \label{fig:franka_setup}
    \end{minipage}
    \hfill
    \begin{minipage}[c]{0.48\linewidth}
        \centering
        \includegraphics[width=\linewidth]{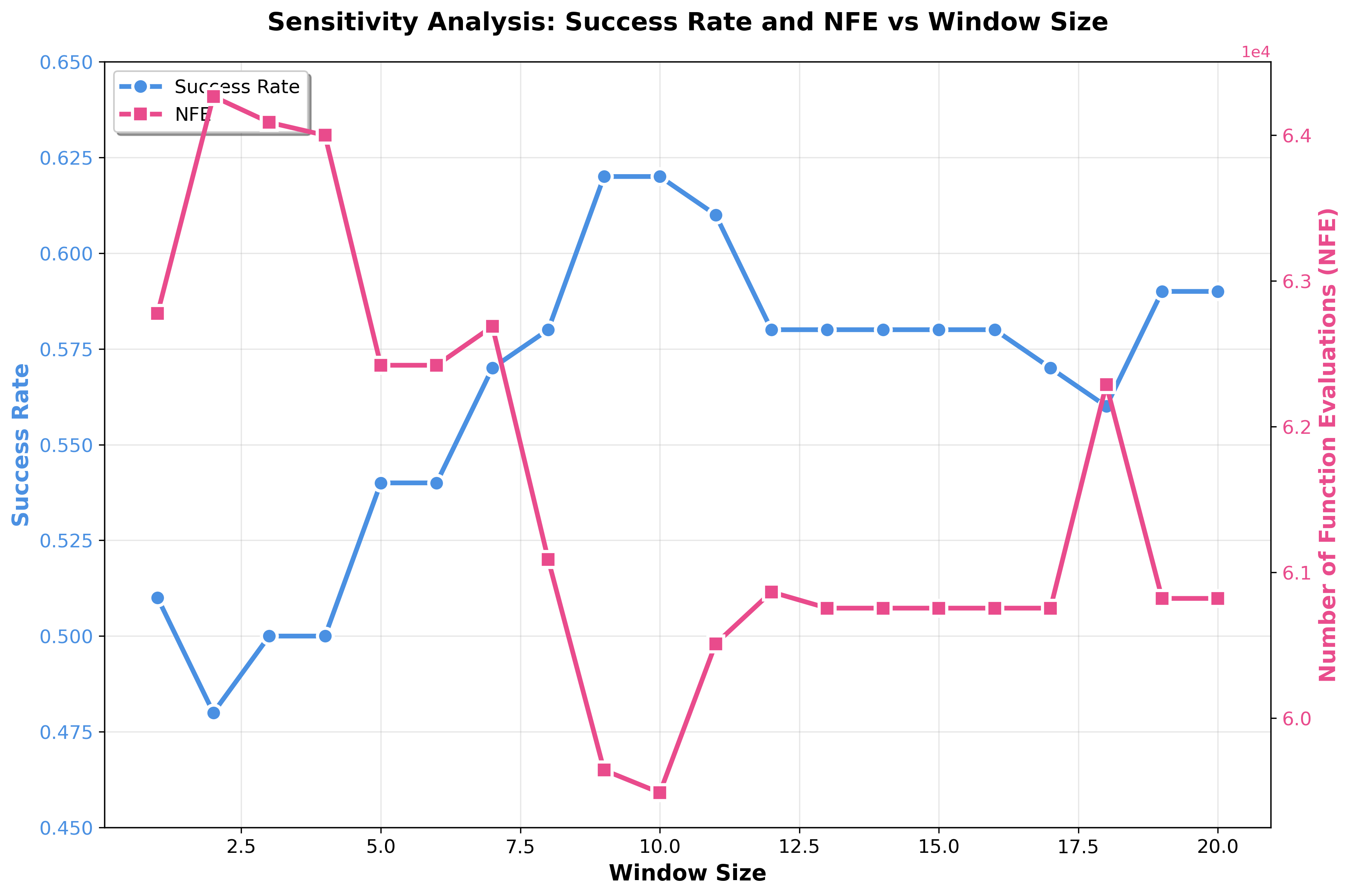}
        \caption{Sensitivity analysis on  window $W$. Small $W$ (1--5) leads to erratic replanning; large $W$ ($>$15) makes the system sluggish. A wide range ($W \in [8, 16]$) yields strong and stable performance.}
        \label{fig:window_sensitivity}
    \end{minipage}
\end{figure}

\paragraph{\method is robust to the smoothing window.}
The choice of smoothing window $W$ affects estimate stability. Small $W$ (e.g., 1--5) yields noisy estimates and erratic replanning; large $W$ (e.g., $>$15) makes the system sluggish. As shown in \cref{fig:window_sensitivity}, \method performs well across a wide range ($W \in [8, 16]$).%, so tuning $W$ is not critical in practice.

\section{Conclusion}
We presented \method, an adaptive replanning layer for neural world-model predictive control. By dynamically contracting the reuse threshold based on model mismatch and local sensitivity, \method achieves substantial planner-side computation reduction while maintaining task performance across image-based and latent planning, as well as  real-world experiments with a Franka robot. These findings position replanning cadence as a critical deployment decision rather than a static parameter.

{
\bibliographystyle{splncs04}
\bibliography{references}
}

%%%%%%%%%%###################################

\appendix
\section{Overview}
This supplementary material is organized to keep the mathematical development contiguous and to move all empirical setup and figures after the theory. \Cref{sec:supp:theoryalg} first collects the shared predictive-control preliminaries, controller pseudocodes, and detailed theoretical results. \Cref{sec:supp:protocol} then summarizes the benchmark setup, transfer protocol, real-world state-based modeling, and corruption settings. \Cref{sec:supp:experiments} finally gathers the additional empirical figures and the exact per-task real-world outcomes underlying the aggregate summaries in the main text.

% The notation follows the main text. In particular, $L_t$ denotes the local Lipschitz sensitivity of the true dynamics, $d_t$ denotes the current-plan deviation, $\widehat L_t$ denotes the online sensitivity proxy in the monitored representation, and $s_t$ denotes the stale-mismatch magnitude used in the proofs. The fixed baselines $\MPC_k^m$ and $\MPC_{k,\epsilon}$ serve as the non-adaptive reference policies. We also distinguish the instantaneous threshold $\tilde\epsilon_t$ used in the theory from the smoothed practical threshold $\epsilon_t$.

\section{Theory and Algorithm Pseudocodes}
\label{sec:supp:theoryalg}
This section gathers the notation, pseudocodes, and proofs in one place. 

\subsection{Predictive-control Preliminaries}
\label{sec:supp:prelims}
This section makes the supplementary material self-contained by recording the shared finite-horizon planning problem, replanning notation, and online quantities used throughout the algorithms and proofs.

\paragraph{True system and cumulative cost.}
For the detailed analysis we make explicit the exogenous parameter sequence hidden inside the compact main-text notation. The executed system evolves as
\begin{equation}
\label{eq:supp_true_system}
x_{t+1}=g_t(x_t,u_t;\xi_t^*),
\qquad
\cost(x_{0:T},u_{0:T-1})\coloneqq \sum_{t=0}^{T-1} f_t(x_t,u_t;\xi_t^*) + F_T(x_T;\xi_T^*).
\end{equation}
The main text suppresses the parameter arguments and writes the same true dynamics and cumulative cost more compactly. In that notation, the learned predictive model $\hat g_t$ is the planner-side realization of the same parameterized family: execution uses the realized parameters $\xi_t^*$, whereas planning uses predicted parameters $\xi_{\tau\mid t}$.

\paragraph{Finite-horizon planning problem.}
For any interval $[t_1,t_2]$, initial state $z$, model-side parameter sequence $\xi_{t_1:t_2-1}$, terminal parameter $\zeta_{t_2}$, and terminal objective $F$, the planner solves
\begin{align}
\label{eq:supp_ftocp}
\min_{y_{t_1:t_2},\,v_{t_1:t_2-1}}\quad
&\sum_{\tau=t_1}^{t_2-1} f_\tau(y_\tau,v_\tau;\xi_\tau) + F(y_{t_2};\zeta_{t_2})\\
\text{s.t.}\quad
&y_{\tau+1}=g_\tau(y_\tau,v_\tau;\xi_\tau),\qquad t_1\le \tau < t_2,\nonumber\\
&y_{t_1}=z.\nonumber
\end{align}
Path constraints can be included in the feasible set or absorbed into indicator costs without changing the notation below. We write
\[
\psi_{t_1}^{t_2}(z,\xi_{t_1:t_2-1},\zeta_{t_2};F)=(y_{t_1:t_2},v_{t_1:t_2-1})
\]
for any optimal solution, and abbreviate the argument list to $\psi_{t_1}^{t_2}(z,\xi;F)$ when the terminal parameter is implicit from context. When the shortened-horizon planner uses the same target-conditioned terminal objective family but with different terminal targets, we keep the family symbol $F$ fixed and vary only the terminal parameter $\zeta_{t_2}$.

\paragraph{Receding-horizon execution and cached plans.}
At online time $t$, the controller solves \eqref{eq:supp_ftocp} on $[t,t^+]$ with $t^+=\min(t+k,T)$ from the current state $x_t$ using the current model-side prediction sequence $\xi_{t:t^+-1\mid t}$. The resulting plan is denoted by $(y_{t:t^+\mid t},v_{t:t^+-1\mid t})$. Standard step-wise MPC executes only the first action $u_t=v_{t\mid t}$ and replans at the next step. If a plan is reused, $p(t)$ denotes the most recent replanning time, and $\hat y_t\equiv y_{t\mid p(t)}$ is the cached nominal state aligned with the current execution step. When $t+k<T$, the terminal objective may encode an intermediate target; the corresponding model-selected terminal state is denoted by $\bar y(\xi_{t+k\mid t})$. When the shortened horizon reaches $T$, the planner uses the episode terminal cost $F_T$.

\paragraph{Predictions and the clairvoyant comparator.}
The symbol $\xi_{\tau\mid t}$ denotes the quantity predicted at planning time $t$ for stage $\tau$, whereas $\xi_\tau^*$ is its realized counterpart. The clairvoyant optimum $\OPT$ is the full-horizon solution of \eqref{eq:supp_ftocp} under the true parameter sequence $\xi^*_{0:T}$. Starting from a current state $x_t$, we write
\[
(x^*_{t:T\mid t},u^*_{t:T-1\mid t})\coloneqq \psi_t^T(x_t,\xi_{t:T-1}^*,\xi_T^*;F_T).
\]
The first action on this continuation is $u^*_{t\mid t}$. If a proof conditions on the most recent replanning time $t'$, we write
\[
u^*_{t\mid t'} \,\coloneqq\, \psi_{t'}^T(x_{t'},\xi_{t':T-1}^*,\xi_T^*;F_T)_{u_t},
\]
for the stage-$t$ action of that continuation.

\paragraph{Monitored representation and online quantities.}
The runtime trigger may be evaluated in a monitored representation rather than directly in the physical state. We write
\[
z_t\coloneqq \phi(x_t),
\qquad
\hat z_t\coloneqq \phi(\hat y_t),
\qquad
d_t\coloneqq \norm{z_t-\hat z_t}_2.
\]
The practical controller also forms the online sensitivity proxy
\begin{equation}
\label{eq:supp_Lhat_prelim}
\widehat L_t\coloneqq \frac{\norm{\phi(x_{t+1})-\phi(x_t)}_2}{\norm{u_t}_2+\gamma},
\end{equation}
using the same monitored representation as the trigger. The planner state $x_t$ and monitored representation $z_t=\phi(x_t)$ need not coincide: in VP2, $x_t$ is image-space while $z_t$ is a frozen image-feature embedding; in TD-MPC2 the planner latent is used for both; and in Franka both coincide with the physical state space.

\paragraph{Controller family and performance metrics.}
All four controllers in this paper solve the same horizon-$k$ planning problem \eqref{eq:supp_ftocp}; they differ only in the replanning schedule. $\MPC_k^1$ replans every step, $\MPC_k^m$ commits a fixed number of actions before refreshing, $\MPC_{k,\epsilon}$ replans whenever $d_t>\epsilon$ or the cached suffix expires, and \method uses the same trigger with a time-varying threshold $\epsilon_t$. The main text abbreviates the standard step-wise controller as $\MPC_k \equiv \MPC_k^1$. We evaluate these controllers by dynamic regret $\mathrm{Regret}(\mathrm{ALG})=\cost(\mathrm{ALG})-\cost(\OPT)$ and by the number of function evaluations (NFE), namely the total number of world-model queries issued by the planner during one episode.

\subsection{Algorithm Pseudocodes}
\label{sec:supp:algorithms}
All three algorithms below invoke the same horizon-$k$ planning problem from \Cref{sec:supp:prelims}; only the rule that refreshes the current plan changes across controllers.

\begin{algorithm}[t]
\caption{Model predictive control with step-wise replanning ($\MPC_k^1$)}
\label{alg:supp_mpc}
\begin{algorithmic}[1]
\Require Prediction horizon $k$, initial state $x_0$, predictive model, planner, terminal cost rule
\For{$t = 0,1,\ldots,T-1$}
    \State Observe the current state $x_t$
    \State Solve the horizon-$k$ planning problem to obtain $(y_{t:t+k}, v_{t:t+k-1})$
    \State Execute the first action $u_t \gets v_{t\mid t}$
\EndFor
\end{algorithmic}
\end{algorithm}

\begin{algorithm}[t]
\caption{Model predictive control with fixed-step reuse ($\MPC_k^m$)}
\label{alg:supp_mpcfr}
\begin{algorithmic}[1]
\Require Prediction horizon $k$, reuse interval $m$, initial state $x_0$, predictive model, planner
\State $t \gets 0$
\While{$t < T$}
    \State Observe the current state $x_t$
    \State Solve the horizon-$k$ planning problem to obtain $(y_{t:t+k}, v_{t:t+k-1})$
    \State Execute the first $m_{\mathrm{commit}}=\min(m,T-t)$ actions of the current plan
    \State $t \gets t + m_{\mathrm{commit}}$
\EndWhile
\end{algorithmic}
\end{algorithm}

\begin{algorithm}[t]
\caption{Model predictive control with fixed-threshold reuse ($\MPC_{k,\epsilon}$)}
\label{alg:supp_mpcft}
\begin{algorithmic}[1]
\Require Prediction horizon $k$, threshold $\epsilon$, initial state $x_0$, predictive model, planner
\State $t \gets 0$, $t_{\mathrm{plan}} \gets 0$
\While{$t < T$}
    \If{$t = t_{\mathrm{plan}}$}
        \State Solve the horizon-$k$ planning problem and cache $(\hat y_{t:t+k}, v_{t:t+k-1})$
    \EndIf
    \State Execute the current suffix action $u_t \gets v_{t\mid t_{\mathrm{plan}}}$ and observe $x_{t+1}$
    \State Compute the next-step deviation $d_{t+1} \gets \norm{\phi(x_{t+1})-\phi(\hat y_{t+1})}_2$
    \If{$d_{t+1} > \epsilon$ \textbf{or} $t+1$ reaches the end of the cached plan}
        \State $t_{\mathrm{plan}} \gets t+1$
    \EndIf
    \State $t \gets t+1$
\EndWhile
\end{algorithmic}
\end{algorithm}
\subsection{Detailed Theoretical Results}
\label{sec:supp:theory}
\subsubsection{Assumptions and notation}
The detailed theory follows a simple chain. \Cref{lem:supp_conditional,lem:supp_state_dev,lem:supp_regret_conditional} derive the fixed-step regret bound in \Cref{thm:supp_fixed_step}. \Cref{lem:supp_perstep_eps,cor:supp_compact_perstep} summarize threshold-based stale-plan error in the compact form used in the main text. \Cref{thm:supp_fixed_eps,prop:supp_adaptive_penalty,thm:supp_adaptive_direct,cor:supp_adaptive_regret} then specialize this picture to the fixed-threshold and adaptive controllers.

Our analysis follows the perturbation framework used in regret analyses of predictive control under model mismatch \cite{lin2022bounded,lin2021perturbation}. We reuse the common planner notation from \Cref{sec:supp:prelims} and introduce only the additional regularity, perturbation, and proof-only mismatch objects needed for the detailed regret bounds.
\FloatBarrier

\newpage
\paragraph{Standard regularity assumptions.}
We assume throughout that:
\begin{itemize}
    \item \textbf{Bounded clairvoyant optimum.} The clairvoyant optimum is bounded: there exists $D_{x^*}>0$ with $\norm{x_t^*}\leq D_{x^*}$ for all $t$.
    \item \textbf{Lipschitz dynamics.} The ground-truth dynamics $g_t(\cdot,\cdot;\xi_t^*)$ are Lipschitz in state and action: there exists $L_t>0$ such that
    \begin{equation}
    \label{eq:supp_dynamics_lip}
    \norm{g_t(x_t,u_t;\xi_t^*)-g_t(x_t',u_t';\xi_t^*)} \leq L_t\big(\norm{x_t-x_t'}+\norm{u_t-u_t'}\big).
    \end{equation}
    Below, we suppress the realized parameter $\xi_t^*$ inside $g_t$ whenever no ambiguity arises.
    \item \textbf{Smooth costs.} Each stage cost and the terminal cost are non-negative and $\ell$-smooth in their arguments \cite{lin2021perturbation,lin2022bounded}.
\end{itemize}

We also write
\[
L \coloneqq \max_{0\le t < T} L_t,
\qquad
L^{(m)} \coloneqq \max_{0\le t\le T-k}\max_{0\le i\le m-1}\prod_{s=t}^{t+i}L_s,
\]
so that $L$ captures the largest single-step local sensitivity and $L^{(m)}$ captures the largest cumulative sensitivity over an $m$-step reuse block. In particular, $L^{(m)}\leq L^m$.

\paragraph{Perturbation bounds.}
Let $\psi_{t_1}^{t_2}(z,\xi;F)$ denote the finite-horizon optimal-control solution from \Cref{eq:supp_ftocp}, initialized at state $z$ with prediction sequence $\xi$ and terminal objective $F$. We assume the planner admits the standard local perturbation bounds used in predictive-control regret analyses.

For \emph{parameter perturbations} with fixed initial state,
\begin{equation}
\label{eq:supp_perturb_initial}
\norm{\psi_{t_1}^{t_2}(z,\xi;F)_{v_t}-\psi_{t_1}^{t_2}(z,\xi';F)_{v_t}}
\leq \Big(\sum_{s=t_1}^{t_2} q_1(s-t_1)\delta_s\Big)\norm{z} + \sum_{s=t_1}^{t_2} q_2(s-t_1)\delta_s,
\end{equation}
where $\delta_s\coloneqq \norm{\xi_s-\xi_s'}$, and the decay kernels satisfy $\sum_{t\ge 0} q_i(t)\le C_i$ for $i\in\{1,2\}$.

For \emph{initial-state perturbations} with fixed parameters,
\begin{equation}
\label{eq:supp_perturb_state}
\norm{\psi_{t_1}^{t_2}(z,\xi;F)_{y_t/v_t}-\psi_{t_1}^{t_2}(z',\xi;F)_{y_t/v_t}}
\le q_3(t-t_1)\norm{z-z'},
\end{equation}
where $\sum_{t\ge 0} q_3(t)\le C_3$.

\paragraph{Local validity region.}
As in the original analysis, we only require these perturbation bounds to hold inside a tube around the clairvoyant optimum. Concretely, there exists $R_1>0$ such that \eqref{eq:supp_perturb_initial} and \eqref{eq:supp_perturb_state} hold whenever the relevant initial states remain in $\mathcal{B}(x_t^*,R_1)$ and the terminal target chosen by the planner stays in a comparable neighborhood of the optimal terminal state. This is the detailed form of the ``local perturbation assumptions'' referred to in the main text.

\paragraph{Prediction parameters and terminal targets.}
We follow the prediction notation from \Cref{sec:supp:prelims}. The only additional convention needed in the proofs is that $\bar y(\xi_{t+k\mid t})$ denotes the terminal state selected by the planner's terminal objective at planning time $t$; by assumption it lies inside the local perturbation tube around the corresponding clairvoyant terminal state.

\paragraph{Monitored representation and trigger space.}
The online trigger and sensitivity proxy from \Cref{sec:supp:prelims} may be evaluated in a monitored representation $z=\phi(x)$ rather than directly in the physical state. To connect $d_t=\norm{\phi(x_t)-\phi(\hat y_t)}_2$ to the state-deviation terms used by the perturbation lemmas, we assume that inside the local tube there exist constants $0<c_\phi\leq C_\phi$ such that
\begin{equation}
\label{eq:supp_phi_bilip}
c_\phi \norm{x-x'} \leq \norm{\phi(x)-\phi(x')} \leq C_\phi \norm{x-x'}
\end{equation}
for all relevant state pairs. Equivalently, one may read the theory below as stated directly in the monitored planning space. When $\phi$ is the identity, \eqref{eq:supp_phi_bilip} is immediate with $c_\phi=C_\phi=1$.

\paragraph{Prediction error and per-step errors.}
The prediction error made at time $t$ for lead time $\tau$ is
\begin{equation}
\rho_{t,\tau} \coloneqq \norm{\xi_{t+\tau\mid t}-\xi^*_{t+\tau}}.
\end{equation}
Building on the comparator notation from \Cref{sec:supp:prelims}, the per-step control error of an online controller is
\begin{equation}
 e_t \coloneqq \norm{u_t-u^*_{t\mid t}},
\end{equation}
and for fixed-step reuse it is convenient to condition on the most recent replanning time $t'$, in which case
\begin{equation}
 e_{t\mid t'} \coloneqq \norm{u_t-u^*_{t\mid t'}}.
\end{equation}

\paragraph{Mismatch statistic used in the condensed bounds.}
We summarize the weighted lead-time mismatch since the last replanning step by
\begin{equation}
 s_t \coloneqq \left(\sum_{\tau=0}^{k-1} \omega_\tau\rho_{p(t),\tau}^2 + c_k^2\right)^{1/2},
\end{equation}
and define the threshold-analysis shorthand
\begin{equation}
E \coloneqq \sum_{t=0}^{T-1}\big(s_t+s_t^2\big).
\label{eq:supp_E_definition}
\end{equation}
The linear part is convenient because the fixed-threshold derivation contains the mixed term $\epsilon\sum_t s_t$, while the quadratic part is the natural mismatch energy. The coefficients $\omega_\tau$ and the terminal term $c_k$ are fixed constants induced by the perturbation decomposition for the chosen horizon; they are not algorithmic parameters and are never tuned in deployment. Their role is only to compress the multi-step mismatch terms into a single shorthand. When a plan is cached at time $t'$, we write $\xi_{t:t'+k-1\mid t'}$ for the model-side parameter sequence reused by that cached plan, and $\bar y(\xi_{t'+k\mid t'})$ for the terminal target predicted by that cached rollout. These quantities appear only in the perturbation lemmas below; the practical controller itself monitors only the online signals $d_t$ and $\widehat L_t$.

\paragraph{Proof strategy and relation to the online trigger.}
The proofs below adapt the perturbation-based dynamic-regret route of \cite{lin2021perturbation,lin2022bounded} to cached-plan reuse. Reusing a plan with tolerance $\varepsilon_t$ incurs per-step execution error of order $O(\varepsilon_t+s_t)$, where $s_t$ summarizes the mismatch accumulated under the cached plan. Those per-step deviations are then passed through the same regret reduction used in prior predictive-control analyses. Because $s_t$ depends on the cached rollout's multi-step prediction errors $\rho_{p(t),\tau}$ and on fixed constants folded into its definition, the controller does not evaluate it directly; it monitors the observable deviation $d_t$ instead. Once stale-plan error is propagated through the dynamics, the threshold-dependent regret terms take the form $L_t^2\varepsilon_t s_t$ and $L_t^2\varepsilon_t^2$. The exponential threshold is chosen because the elementary bounds $xe^{-ax}\le (ea)^{-1}$ and $x^2e^{-ax}\le 4(e^2a^2)^{-1}$ convert these terms into uniform bounds controlled by $(\epsilon_0,\alpha_\delta,\alpha_L)$.

\subsubsection{Detailed theorem statements}
\begin{lemma}[Conditional per-step error bound for fixed-step reuse]
\label{lem:supp_conditional}
Assume the local perturbation bounds hold and let $t'=mn\leq t < m(n+1)$. Suppose the state at the last replanning time satisfies $x_{t'}\in\mathcal{B}(x_{t'}^*,R_1)$, and the terminal target used inside the planner lies in $\mathcal{B}(x_{t'+k}^*,R_2)$ for some $R_2\geq R_1$. Then the conditional per-step error of $\MPC_k^m$ satisfies
\begin{equation}
\label{eq:supp_conditional_bound}
 e_{t\mid t'} \leq \sum_{\tau=0}^{k-1}\big((R_1+D_{x^*})q_1(\tau)+q_2(\tau)\big)\rho_{t',\tau}
 + 2R_2\big((R_1+D_{x^*})q_1(k)+q_2(k)\big).
\end{equation}
\end{lemma}

\begin{proof}
For $t'=mn \leq t < m(n+1)$, the conditional error can be written as the difference between the action supplied by the plan optimized at time $t'$ and the clairvoyant action for the same stage. By the principle of optimality, the relevant portion of the clairvoyant plan is a sub-trajectory of the horizon-$k$ planning problem anchored at $t'$. Applying the parameter-perturbation bound therefore yields
\begin{align}
 e_{t\mid t'}
 &\leq \sum_{\tau=0}^{k-1}\big(\norm{x_{t'}}q_1(\tau)+q_2(\tau)\big)\rho_{t',\tau} \nonumber\\
 &\quad + \big(\norm{x_{t'}}q_1(k)+q_2(k)\big)\norm{\bar y(\xi_{t'+k\mid t'})-x^*_{t'+k\mid t'}}.
\end{align}
Because $\norm{x_{t'}}\leq R_1+D_{x^*}$ and both terminal states lie in the radius-$R_2$ tube around $x_{t'+k}^*$, the final term is bounded by the constant term in \eqref{eq:supp_conditional_bound}. This proves the claim.
\end{proof}

\begin{lemma}[State deviation bound]
\label{lem:supp_state_dev}
Let $t'=mn\leq t < m(n+1)$ and let $x^*_{\tau\mid t'}$ denote the state at time $\tau$ on the clairvoyant optimal trajectory starting from $x_{t'}$. Under the Lipschitz dynamics assumption,
\begin{equation}
\label{eq:supp_state_dev}
\norm{x_t-x^*_{t\mid t'}} \leq \sum_{\tau=t'}^{t-1} e_{\tau\mid t'} \prod_{s=\tau+1}^{t-1} L_s.
\end{equation}
\end{lemma}

\begin{proof}
The claim is proved by induction. For $t=t'+1$,
\[
\norm{x_{t'+1}-x^*_{t'+1\mid t'}} = \norm{g_{t'}(x_{t'},u_{t'})-g_{t'}(x_{t'},u^*_{t'\mid t'})}
\leq L_{t'} e_{t'\mid t'}.
\]
Assume the result holds at time $t-1$. Then, using \eqref{eq:supp_dynamics_lip},
\begin{align}
\norm{x_t-x^*_{t\mid t'}}
&= \norm{g_{t-1}(x_{t-1},u_{t-1})-g_{t-1}(x^*_{t-1\mid t'},u^*_{t-1\mid t'})} \\
&\leq L_{t-1}\Big(\norm{x_{t-1}-x^*_{t-1\mid t'}} + \norm{u_{t-1}-u^*_{t-1\mid t'}}\Big) \\
&\leq \sum_{\tau=t'}^{t-1} e_{\tau\mid t'}\prod_{s=\tau+1}^{t-1} L_s.
\end{align}
This completes the induction.
\end{proof}

\begin{lemma}[Regret from conditional per-step errors]
\label{lem:supp_regret_conditional}
Assume first that $T=mN$ and define
\[
S_m \coloneqq \sum_{i=0}^{N-1}\sum_{\tau=mi}^{m(i+1)-1} e_{\tau\mid mi}^2.
\]
Then the dynamic regret of $\MPC_k^m$ is bounded by
\begin{align}
\cost(\MPC_k^m)-\cost(\OPT)
&\leq \sqrt{\left(\frac{\ell}{2}\cdot 2m\big(L^{(m)}\big)^2C_3^2\right)\cost(\OPT)}\,\sqrt{S_m} \nonumber\\
&\quad + \left(\frac{\ell}{2}\cdot 2m\big(L^{(m)}\big)^2C_3^2\right) S_m.
\label{eq:supp_conditional_regret}
\end{align}
\end{lemma}

\begin{proof}
Combine \cref{lem:supp_conditional,lem:supp_state_dev} to bound the state and action deviations relative to the clairvoyant optimum by weighted sums of conditional per-step errors. After squaring and summing over time, one obtains
\[
A \coloneqq \sum_{t=1}^{T}\norm{x_t-x_t^*}^2 + \sum_{t=0}^{T-1}\norm{u_t-u_t^*}^2
\leq 2m\big(L^{(m)}\big)^2C_3^2\sum_{i=0}^{N-1}\sum_{\tau=mi}^{m(i+1)-1} e_{\tau\mid mi}^2.
\]
For nonnegative $\ell$-smooth stage and terminal costs, the standard smoothness-based dynamic-regret inequality \cite[Lemma 1]{lin2021perturbation} gives
\[
\cost(\MPC_k^m)-\cost(\OPT)
\leq \sqrt{\frac{\ell}{2}\cost(\OPT)A} + \frac{\ell}{2}A.
\]
Substituting the bound on $A$ yields \eqref{eq:supp_conditional_regret}. If $T$ is not divisible by $m$, the same argument is applied to the final partial block of length smaller than $m$; that remainder contributes only lower-order terms absorbed into the same big-$O$ constants used later in \cref{thm:supp_fixed_step}.
\end{proof}

\begin{theorem}[Regret bound for fixed-step reuse]
\label{thm:supp_fixed_step}
Under the assumptions above, if the prediction errors are sufficiently small and the planning horizon is sufficiently long so that the executed trajectory remains in the local tube $\mathcal{B}(x_t^*,R_1)$ for all $t$, then
\begin{equation}
\cost(\MPC_k^m)-\cost(\OPT)
= O\!\left(\sqrt{m\big(L^{(m)}\big)^2\cost(\OPT)E_m}+m\big(L^{(m)}\big)^2E_m\right),
\end{equation}
where
\begin{align}
E_m
&= O\!\Bigg(\sum_{\tau=0}^{k-1}\big((R_1+D_{x^*})q_1(\tau)+q_2(\tau)\big)^2
\sum_{n=0}^{N-1} m\rho_{mn,\tau}^2 \nonumber\\
&\qquad\qquad + \big((R_1+D_{x^*})q_1(k)+q_2(k)\big)^2T\Bigg).
\end{align}
\end{theorem}

\begin{proof}
The proof combines \cref{lem:supp_conditional,lem:supp_regret_conditional}. The small-error condition ensures inductively that the trajectory never leaves the local tube around the clairvoyant optimum, so the perturbation bounds remain valid at every replanning epoch. Squaring \eqref{eq:supp_conditional_bound} yields mismatch terms proportional to $\rho_{mn,\tau}^2$ together with a terminal constant. Summing those terms over all steps executed after replanning time $mn$ produces the stated expression for $E_m$. Substituting this bound into \eqref{eq:supp_conditional_regret} gives the theorem.
\end{proof}

\begin{lemma}[Per-step error for threshold-based reuse]
\label{lem:supp_perstep_eps}
Suppose the current state satisfies $x_t\in\mathcal{B}(x_t^*,R_1)$ and the terminal target used by the planner lies in $\mathcal{B}(x_{t+k}^*,R_2)$. Let $t'$ denote the most recent planning time, and let $\varepsilon_t$ denote a generic nonnegative reuse tolerance. Then the per-step error of both $\MPC_{k,\epsilon}$ and the idealized controller underlying \method is bounded, after absorbing the fixed representation-equivalence constant from \eqref{eq:supp_phi_bilip} when $\phi\neq I$, by
\begin{equation}
\label{eq:supp_perstep_eps}
 e_t \leq q_3(0)\varepsilon_t + \sum_{\tau=0}^{k-1}\big((R_1+D_{x^*})q_1(\tau)+q_2(\tau)\big)\rho_{t',\tau}
 + 2R_2\big((R_1+D_{x^*})q_1(k)+q_2(k)\big).
\end{equation}
\end{lemma}

\begin{proof}
Let $t'=p(t)$ denote the most recent replanning time. The action executed at time $t$ is taken from the plan computed at $t'$, whereas the comparator action is the clairvoyant optimum from the current state. Writing the two quantities explicitly and inserting the cached nominal state $\hat y_{t\mid t'}$ gives
\begin{align}
 e_t
 &= \norm{\psi_{t'}^{t'+k}(x_{t'},\xi_{t':t'+k-1\mid t'},\bar y(\xi_{t'+k\mid t'});F)_{v_t} - \psi_t^{T}(x_t,\xi^*_{t:T};F_T)_{v_t}} \nonumber\\
 &\le \underbrace{\norm{\psi_t^{t'+k}(\hat y_{t\mid t'},\xi_{t:t'+k-1\mid t'},\bar y;F)_{v_t}-\psi_t^{t'+k}(x_t,\xi_{t:t'+k-1\mid t'},\bar y;F)_{v_t}}}_{\text{state-deviation term}} \nonumber\\
 &\quad + \underbrace{\norm{\psi_t^{t'+k}(x_t,\xi_{t:t'+k-1\mid t'},\bar y;F)_{v_t}-\psi_t^{t'+k}(x_t,\xi^*_{t:t'+k-1},x^*_{t+k\mid t};F)_{v_t}}}_{\text{prediction-mismatch term}}.
\end{align}
The first term is bounded by the initial-state perturbation bound \eqref{eq:supp_perturb_state}:
\[
\norm{\psi_t^{t'+k}(\hat y_{t\mid t'},\cdot)_{v_t}-\psi_t^{t'+k}(x_t,\cdot)_{v_t}} \le q_3(0)\norm{\hat y_{t\mid t'}-x_t}.
\]
Because threshold-based reuse continues only while the monitored deviation stays below the active tolerance, $\norm{\phi(\hat y_{t\mid t'})-\phi(x_t)}\le \varepsilon_t$. If the theory is read directly in the monitored planning space this is exactly $\norm{\hat y_{t\mid t'}-x_t}\le \varepsilon_t$; otherwise, the local compatibility condition \eqref{eq:supp_phi_bilip} implies $\norm{\hat y_{t\mid t'}-x_t}\le \varepsilon_t/c_\phi$. This fixed factor is absorbed into the leading tolerance coefficient, contributing the term $q_3(0)\varepsilon_t$ up to constants.

For the second term, apply the parameter-perturbation bound \eqref{eq:supp_perturb_initial} exactly as in \cref{lem:supp_conditional}. Using $\norm{x_t}\le R_1+D_{x^*}$ inside the local tube and the terminal-target assumption $\bar y(\xi_{t+k\mid t})\in \mathcal{B}(x_{t+k}^*,R_2)$ gives
\[
\norm{\psi_t^{t'+k}(x_t,\xi_{t:t'+k-1\mid t'},\bar y;F)_{v_t}-\psi_t^{t'+k}(x_t,\xi^*_{t:t'+k-1},x^*_{t+k\mid t};F)_{v_t}}
\]
\[
\le \sum_{\tau=0}^{k-1}\big((R_1+D_{x^*})q_1(\tau)+q_2(\tau)\big)\rho_{t',\tau}
 + 2R_2\big((R_1+D_{x^*})q_1(k)+q_2(k)\big).
\]
Combining the two pieces yields \eqref{eq:supp_perstep_eps}.
\end{proof}

\begin{corollary}[Compact per-step stale-plan form]
\label{cor:supp_compact_perstep}
Under the assumptions of \cref{lem:supp_perstep_eps}, the stale-plan error of both $\MPC_{k,\epsilon}$ and the idealized controller underlying \method satisfies
\begin{equation}
 e_t = O(\varepsilon_t + s_t),
\end{equation}
where the hidden constants depend only on the perturbation kernels and local-tube radii defined above.
\end{corollary}

\begin{proof}
The term $q_3(0)\varepsilon_t$ in \cref{lem:supp_perstep_eps} is linear in the active tolerance, while the weighted mismatch sum together with the terminal-target term is exactly what is summarized by the mismatch statistic $s_t$. Absorbing the kernel-dependent constants into the $O(\cdot)$ notation gives the compact form.
\end{proof}

\begin{theorem}[Regret bound for fixed-threshold reuse]
\label{thm:supp_fixed_eps}
Under the same assumptions, if the threshold $\epsilon$ and prediction errors are sufficiently small to keep the trajectory inside the local perturbation tube, then
\begin{align}
\cost(\MPC_{k,\epsilon})-\cost(\OPT)
&= O\!\left(\sqrt{L^2\cost(\OPT)(E+\epsilon E+\epsilon^2T)}\right.\nonumber\\
&\qquad\left.+L^2(E+\epsilon E+\epsilon^2T)\right),
\end{align}
where $E$ is the cumulative mismatch budget defined in \eqref{eq:supp_E_definition}.
\end{theorem}

\begin{proof}
Starting from \cref{cor:supp_compact_perstep} with the constant tolerance $\varepsilon_t=\epsilon$, squaring the per-step error gives
\[
e_t^2 = O(\epsilon^2) + O(\epsilon s_t) + O(s_t^2).
\]
Summing over time gives
\[
\sum_{t=0}^{T-1} e_t^2 = O\!\left(\epsilon^2T + \epsilon\sum_{t=0}^{T-1}s_t + \sum_{t=0}^{T-1}s_t^2\right).
\]
By the definition \eqref{eq:supp_E_definition}, both $\sum_t s_t$ and $\sum_t s_t^2$ are bounded by $E$, so
\[
\sum_{t=0}^{T-1} e_t^2 = O(E + \epsilon E + \epsilon^2T).
\]
Applying the same smoothness-based dynamic-regret inequality as above with $A=O(E + \epsilon E + \epsilon^2T)$ gives the theorem.
\end{proof}

\paragraph{Monitor conditions for adaptive thresholding.}
The practical controller does not observe $L_t$ or $s_t$ directly. To make the analytical motivation explicit, we therefore impose idealized monitor-dominance conditions \cite[Section 3.2]{lin2022bounded}: the online monitors dominate the unobserved quantities up to constants $\lambda,\mu>0$, namely $\widehat L_t\geq \lambda L_t$ and $d_t\geq \mu s_t$. These constants quantify monitor quality and appear only inside the hidden constants of the final adaptive regret bound.

\begin{proposition}[Exponential thresholding bounds the stale-plan penalties]
\label{prop:supp_adaptive_penalty}
Assume there exist constants $\lambda,\mu>0$ such that $\widehat{L}_t\geq \lambda L_t$ and $d_t\geq \mu s_t$. Then the instantaneous threshold
\[
\tilde\epsilon_t = \epsilon_0 e^{-\alpha_\delta d_t}e^{-\alpha_L \widehat L_t}
\]
implies
\begin{align}
L_t^2\tilde\epsilon_t^2 &\leq \frac{\epsilon_0^2}{e^2\lambda^2\alpha_L^2},\\
L_t^2\tilde\epsilon_t s_t &\leq \frac{4\epsilon_0}{e^3\lambda^2\mu\alpha_L^2\alpha_\delta}.
\end{align}
\end{proposition}

\begin{proof}
Using $e^{-x}\leq 1/(ex)$ for $x>0$,
\begin{equation}
L_t\tilde\epsilon_t \leq \epsilon_0L_t e^{-\alpha_L\widehat L_t}
\leq \epsilon_0\frac{L_t}{e\alpha_L\widehat L_t}
\leq \frac{\epsilon_0}{e\lambda\alpha_L},
\end{equation}
which gives the first claim after squaring. For the mixed term, use both $e^{-x}\leq 1/(ex)$ and $e^{-x}\leq 4/(e^2x^2)$:
\begin{align}
L_t^2\tilde\epsilon_t s_t
&\leq \epsilon_0\big(L_t^2 e^{-\alpha_L\widehat L_t}\big)\big(e^{-\alpha_\delta d_t}s_t\big) \\
&\leq \epsilon_0 \left(\frac{4L_t^2}{e^2\alpha_L^2\widehat L_t^2}\right)\left(\frac{s_t}{e\alpha_\delta d_t}\right)
\leq \frac{4\epsilon_0}{e^3\lambda^2\mu\alpha_L^2\alpha_\delta}.
\end{align}
This proves the proposition.
\end{proof}

\begin{theorem}[Adaptive regret with time-varying thresholds]
\label{thm:supp_adaptive_direct}
Under the assumptions of \cref{thm:supp_fixed_eps}, if the idealized controller underlying \method executes with the instantaneous threshold $\tilde\epsilon_t$ while the trajectory remains inside the local perturbation tube, then
\begin{align}
\cost(\method)-\cost(\OPT)
&= O\!\Bigg(\sqrt{\cost(\OPT)\Big(L^2E+\sum_{t=0}^{T-1}L_t^2\tilde\epsilon_t s_t 
+ \sum_{t=0}^{T-1}L_t^2\tilde\epsilon_t^2\Big)} \nonumber\\
&\qquad\ +L^2E+\sum_{t=0}^{T-1}L_t^2\tilde\epsilon_t s_t 
+ \sum_{t=0}^{T-1}L_t^2\tilde\epsilon_t^2\Bigg),
\end{align}
where $E$ is the cumulative mismatch budget defined in \eqref{eq:supp_E_definition}.
\end{theorem}

\begin{proof}
Apply \cref{cor:supp_compact_perstep} with the time-varying tolerance $\varepsilon_t=\tilde\epsilon_t$. Squaring the resulting per-step bound gives three contributions: a pure mismatch term, a pure threshold term, and a mixed term,
\[
e_t^2 = O(s_t^2) + O(\tilde\epsilon_t^2) + O(\tilde\epsilon_t s_t).
\]
When the state-deviation inequality is propagated through the Lipschitz dynamics, the last two contributions inherit the same local-sensitivity factors that appear in the main-text statement. Consequently,
\[
\sum_{t=0}^{T-1} e_t^2 = O\!\left(E + \sum_{t=0}^{T-1}L_t^2\tilde\epsilon_t s_t + \sum_{t=0}^{T-1}L_t^2\tilde\epsilon_t^2\right),
\]
because the pure mismatch contribution $\sum_t s_t^2$ is dominated by the budget $E$ from \eqref{eq:supp_E_definition}.
Substituting this expression into the same smoothness-based dynamic-regret inequality used for \cref{thm:supp_fixed_eps} with $A=O\!\left(E + \sum_t L_t^2\tilde\epsilon_t s_t + \sum_t L_t^2\tilde\epsilon_t^2\right)$ gives the stated adaptive regret bound.
\end{proof}

\begin{corollary}[Monitor-controlled adaptive regret]
\label{cor:supp_adaptive_regret}
Under the assumptions of \cref{thm:supp_adaptive_direct,prop:supp_adaptive_penalty}, the regret of \method satisfies
\begin{align}
\cost(\method)-\cost(\OPT)
&= O\!\left(\sqrt{\cost(\OPT)\left(L^2E+\frac{\epsilon_0}{\alpha_L^2}\left(\epsilon_0+\frac{1}{\alpha_\delta}\right)T\right)}\right.\nonumber\\
&\qquad\left.+L^2E+\frac{\epsilon_0}{\alpha_L^2}\left(\epsilon_0+\frac{1}{\alpha_\delta}\right)T\right).
\end{align}
\end{corollary}

\begin{proof}
Invoke \cref{prop:supp_adaptive_penalty} inside \cref{thm:supp_adaptive_direct}. This replaces the two threshold-dependent sums by terms of order $O(\epsilon_0T/\alpha_L^2\alpha_\delta)$ and $O(\epsilon_0^2T/\alpha_L^2)$, respectively. Collecting these contributions yields the stated closed form. The practical controller uses the smoothed threshold as a stabilized implementation of the same monotone update.
\end{proof}

\clearpage
\section{Experimental Details}
\label{sec:supp:protocol}

\subsection{Real-World Experiments Setup}
\paragraph{Real-world Franka manipulation.}
The physical experiments use a Franka Emika Panda arm equipped with a cubic pusher end-effector and a Vicon motion-capture system (\cref{fig:supp_real_setup}). Door opening is represented with a 6-dimensional state containing the hinge, handle, and end-effector keypoints; the corresponding learned world model is a 3-layer MLP. T-block manipulation is represented with an 8-dimensional state containing three object keypoints together with the end-effector; the corresponding world model is a 6-layer MLP. We report binary task success and NFE for two task families: door opening to $90^\circ$ and $180^\circ$, and T-block translation, rotation, and combined translation-plus-rotation. Each physical task instance is evaluated for 10 trials, yielding 20 articulation trials and 30 rearrangement trials, for 50 physical trials in total.

\begin{figure}[H]
    \centering
    \includegraphics[width=0.78\linewidth]{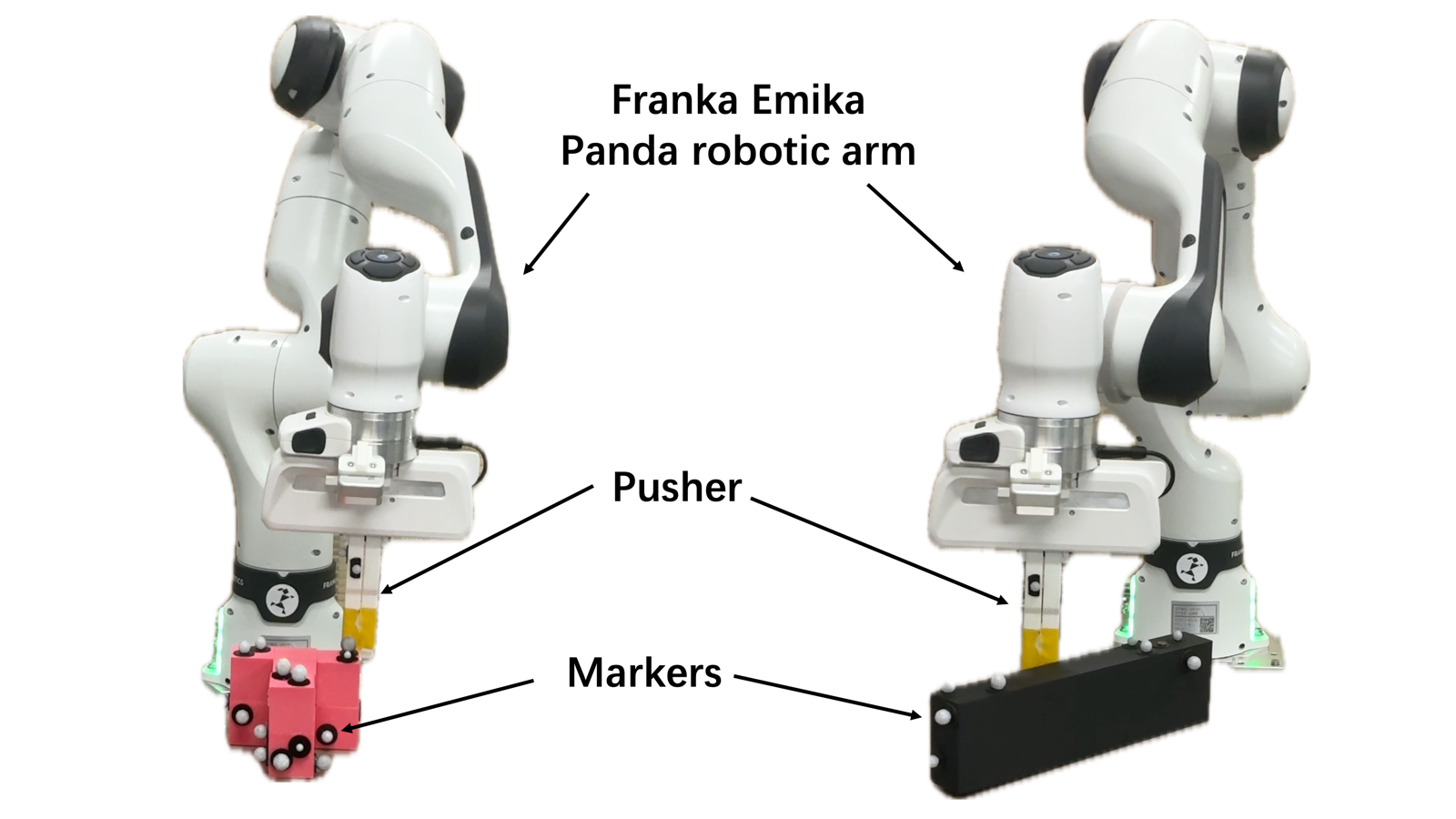}
    \caption{Real-world Franka experimental platform used for the articulation and T-block manipulation tasks. The robot is equipped with a cubic pusher end-effector and tracked with a Vicon motion-capture system.}
    \label{fig:supp_real_setup}
\end{figure}

\subsection{State-based World Modeling for Real-World Tasks}
\paragraph{Real-world state representation and training.}
The state-based Franka models are learned as one-step predictors $x_{t+1}=f(x_t,u_t)$ with $u_t\in\mathbb{R}^2$ the planar pusher command. For door opening, the state comprises the 2D hinge, handle, and end-effector keypoints. For T-block manipulation, the state contains three object keypoints together with the end-effector (\cref{fig:supp_franka_state_repr}). Training data are collected from a mixture of successful demonstrations and random exploration trajectories. The state-based models are trained with Adam, learning rate $5\times 10^{-6}$, batch size 16, for 300 epochs, and multi-step rollouts are obtained autoregressively at deployment time. The optimizer settings are therefore fixed across the two physical backbones; only the learned predictor differs between the door and T-block state representations.

\begin{figure}[t]
    \centering
    \includegraphics[width=0.88\linewidth]{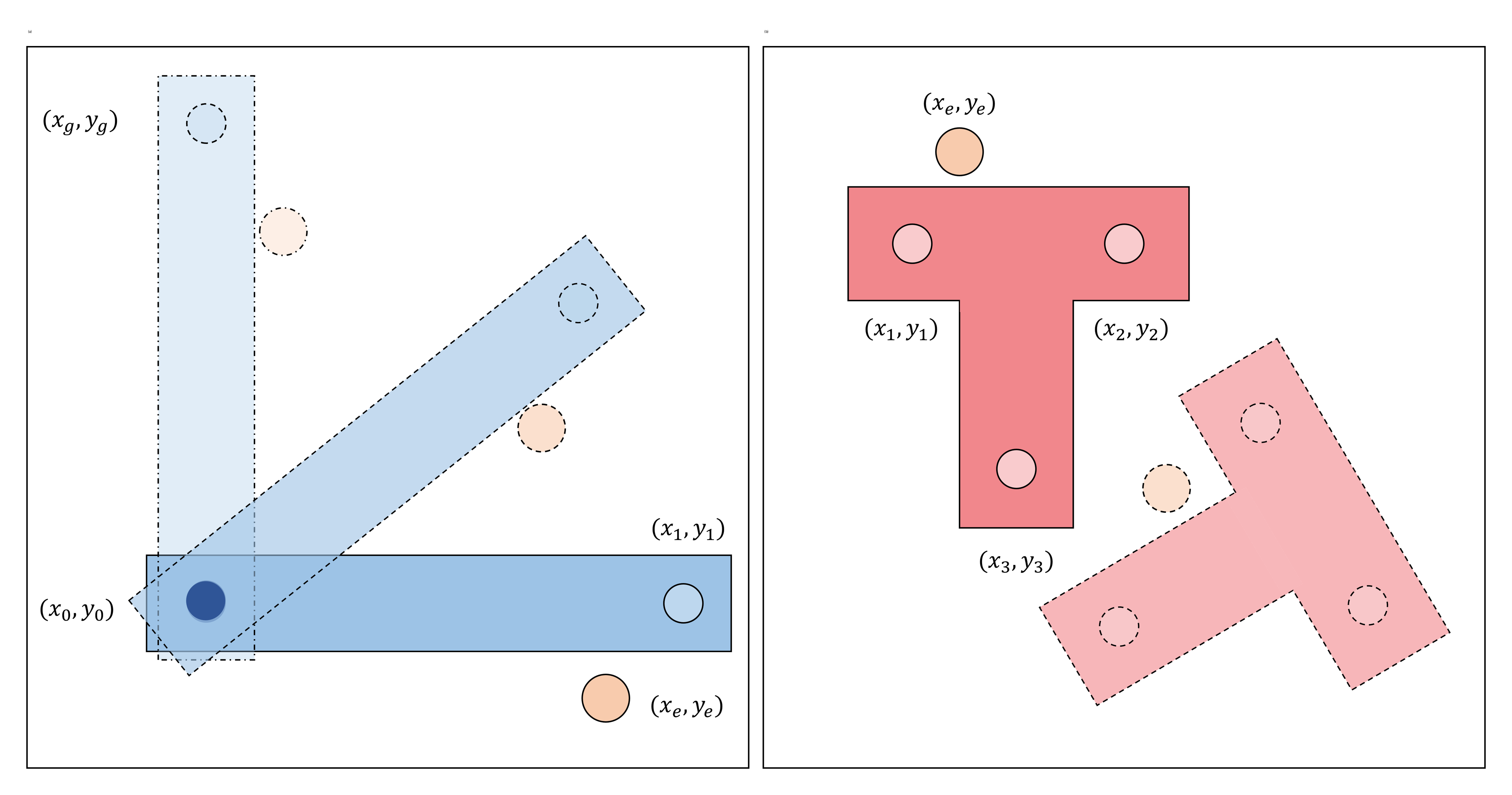}
    \caption{State representations used for the real-world Franka experiments. Panel (a) shows the door-opening representation with hinge, handle, and end-effector keypoints. Panel (b) shows the T-block representation with three object keypoints and the end-effector.}
    \label{fig:supp_franka_state_repr}
\end{figure}

\subsection{Disturbed Simulators and Visual Corruptions}
\paragraph{Disturbed simulator configuration.}
For the disturbance study, Gaussian noise with mean zero is added to selected state components. The standard deviations used for the three disturbance levels are listed in \cref{tab:noisy_simulator_cfg}.

\begin{table}[t]
    \centering
    \caption{Standard deviations of Gaussian noise applied to state components in the disturbed-simulator experiments.}
    \label{tab:noisy_simulator_cfg}
    \begin{tabular}{lccc}
    \toprule
    Component & Level 1 & Level 2 & Level 3 \\
    \midrule
    Robot position & $0.001$ & $0.005$ & $0.010$ \\
    Robot velocity & $0.001$ & $0.005$ & $0.010$ \\
    Object position & $0.000$ & $0.001$ & $0.005$ \\
    Object velocity & $0.000$ & $0.001$ & $0.005$ \\
    End-effector position & $0.001$ & $0.005$ & $0.010$ \\
    \bottomrule
    \end{tabular}
\end{table}
\paragraph{Visual corruption parameters.}
For the image-space corruption study, we apply additive Gaussian noise or Gaussian blur to the predicted images before feature extraction. The parameter settings are summarized in \cref{tab:visual_noise_params,tab:visual_blur_params}, and representative examples are shown in \cref{fig:visual_corruption_examples}.

\begin{table}[t]
\centering
\begin{minipage}{.48\linewidth}
    \centering
    \caption{Gaussian-noise parameters for predicted images.}
    \label{tab:visual_noise_params}
    \begin{tabular}{cc}
    \toprule
    Noise level & std \\
    \midrule
    Level 1 & 0.1 \\
    Level 2 & 0.2 \\
    Level 3 & 0.5 \\
    \bottomrule
    \end{tabular}
\end{minipage}\hfill
\begin{minipage}{.48\linewidth}
    \centering
    \caption{Gaussian-blur parameters for predicted images.}
    \label{tab:visual_blur_params}
    \begin{tabular}{cc}
    \toprule
    Blur level & sigma \\
    \midrule
    Level 1 & 0.1 \\
    Level 2 & 0.5 \\
    Level 3 & 1.0 \\
    \bottomrule
    \end{tabular}
\end{minipage}
\end{table}

\begin{figure}[H]
    \centering
    \includegraphics[width=0.94\linewidth]{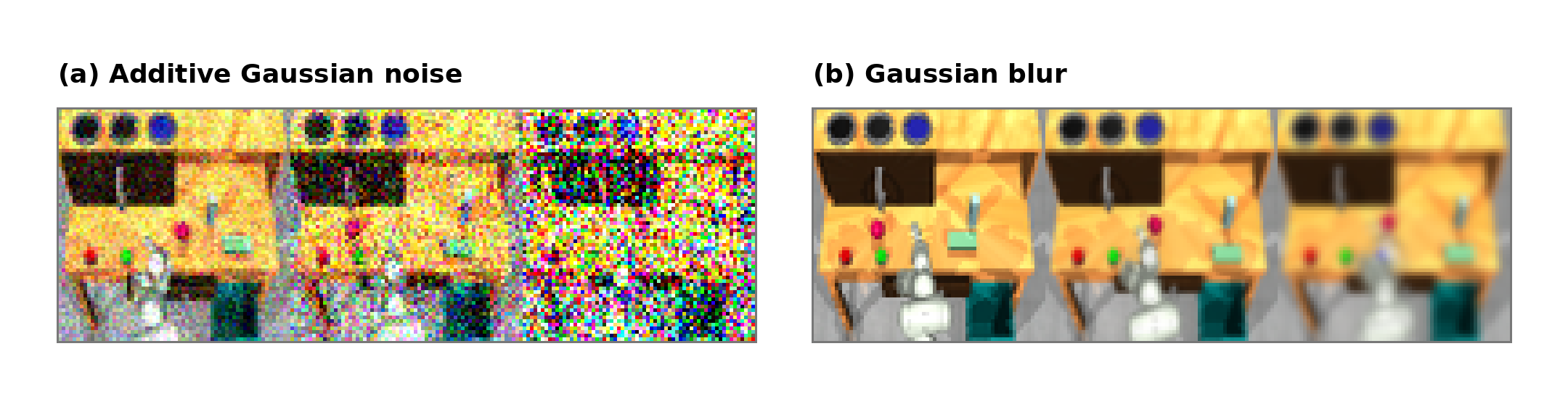}
    \caption{Representative visual corruptions used in the image-space robustness study. The left panel shows additive Gaussian noise applied at increasing levels, and the right panel shows increasing Gaussian blur.}
    \label{fig:visual_corruption_examples}
\end{figure}

\FloatBarrier
\section{More Experimental Results}
\label{sec:supp:experiments}
This section groups the additional empirical material after the setup section. It contains the raw fixed-baseline sweeps used to define the operating range, the diagnostics that motivate the adaptive rule, and the exact per-task real-world outcomes underlying the aggregate summaries in the main text.

\subsection{Experimental Results for $\MPC_k^m$ and $\MPC_{k,\epsilon}$}
We first show the raw fixed-baseline sweeps for the VP2 benchmark. Keeping these plots together makes the operating-range story easier to follow: different tasks and backbones prefer different fixed reuse settings, which is exactly the variability that motivates an adaptive replanning rule.

\begin{figure}[t]
    \centering
    \begin{minipage}[t]{0.49\linewidth}
        \centering
        \includegraphics[width=\textwidth]{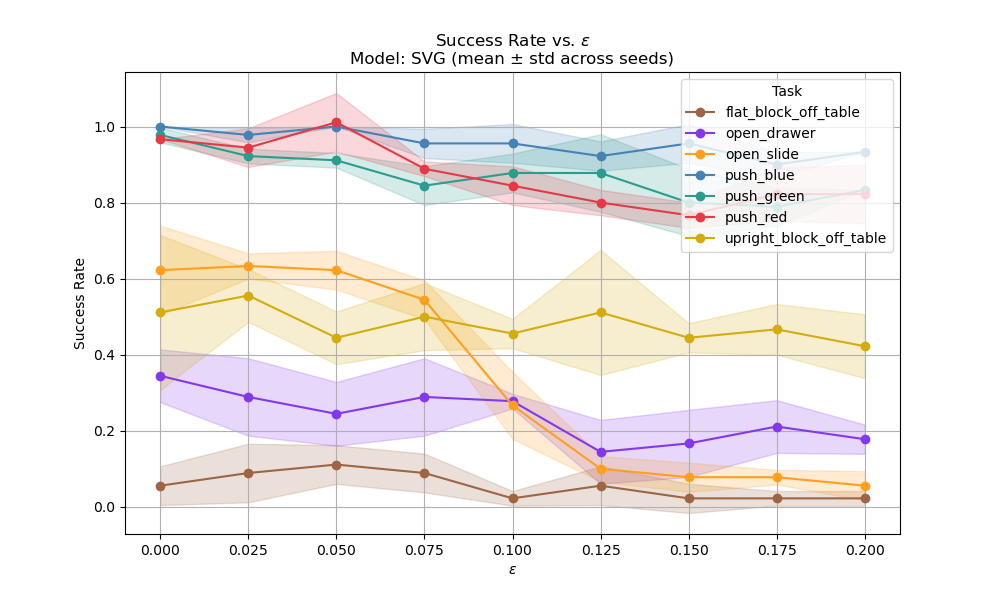}
    \end{minipage}\hfill
    \begin{minipage}[t]{0.49\linewidth}
        \centering
        \includegraphics[width=\textwidth]{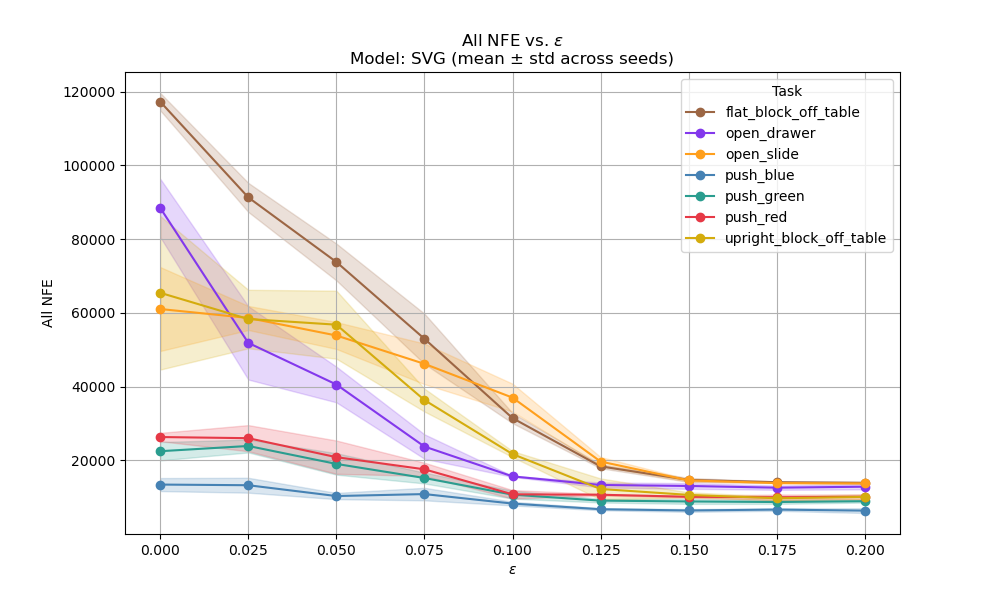}
    \end{minipage}
    \caption{VP2 fixed-threshold sweeps for the SVG world model. Left: task-wise success versus the deviation threshold $\epsilon$. Right: the corresponding NFE versus $\epsilon$.}
    \label{fig:supp_threshold_svg}
\end{figure}

\begin{figure}[t]
    \centering
    \begin{minipage}[t]{0.49\linewidth}
        \centering
        \includegraphics[width=\textwidth]{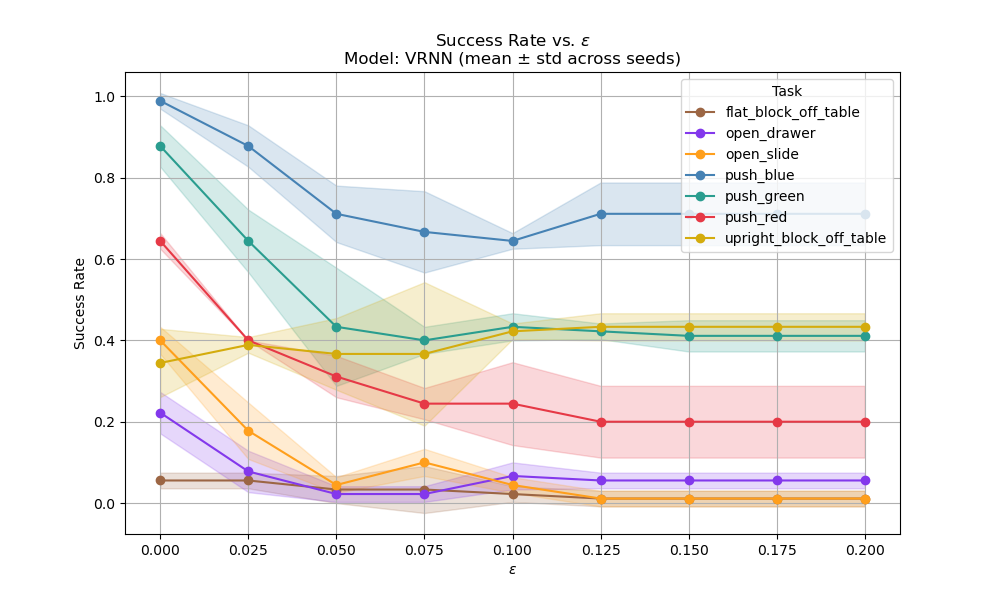}
    \end{minipage}\hfill
    \begin{minipage}[t]{0.49\linewidth}
        \centering
        \includegraphics[width=\textwidth]{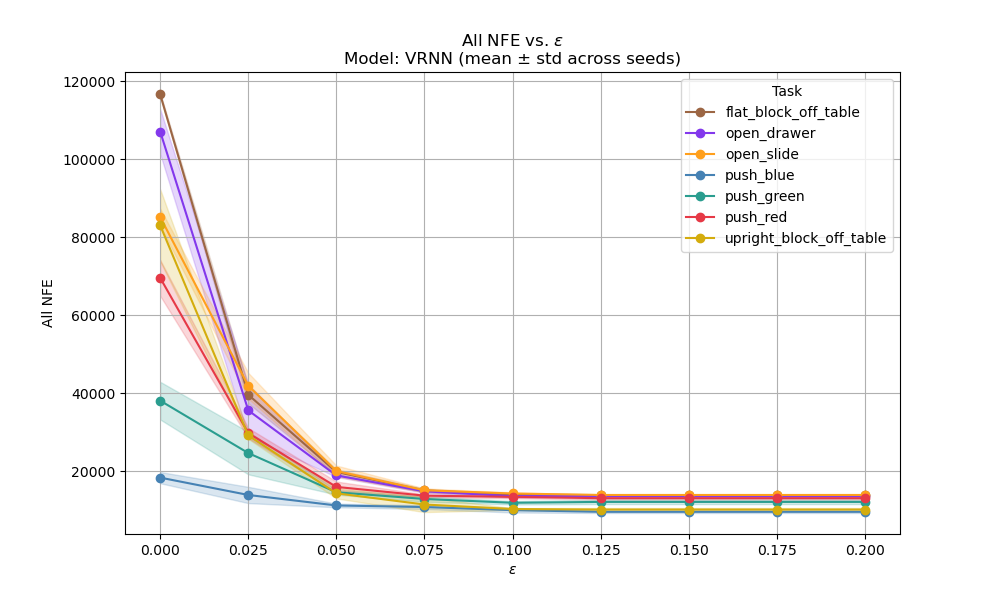}
    \end{minipage}
    \caption{VP2 fixed-threshold sweeps for the Struct-VRNN world model. Left: task-wise success versus $\epsilon$. Right: the corresponding NFE versus $\epsilon$.}
    \label{fig:supp_threshold_vrnn}
\end{figure}

\begin{figure}[t]
    \centering
    \begin{minipage}[t]{0.49\linewidth}
        \centering
        \includegraphics[width=\textwidth]{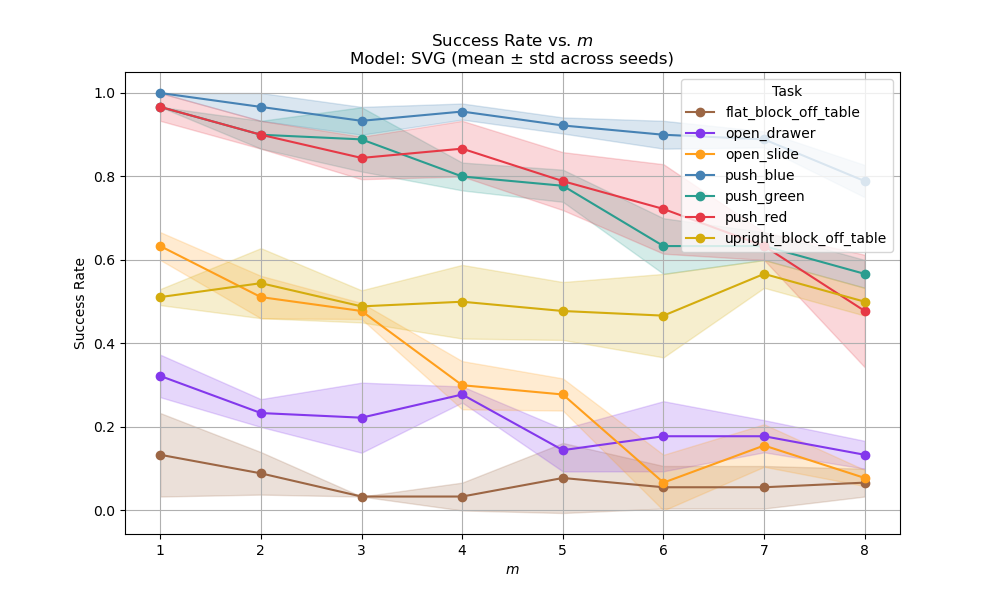}
    \end{minipage}\hfill
    \begin{minipage}[t]{0.49\linewidth}
        \centering
        \includegraphics[width=\textwidth]{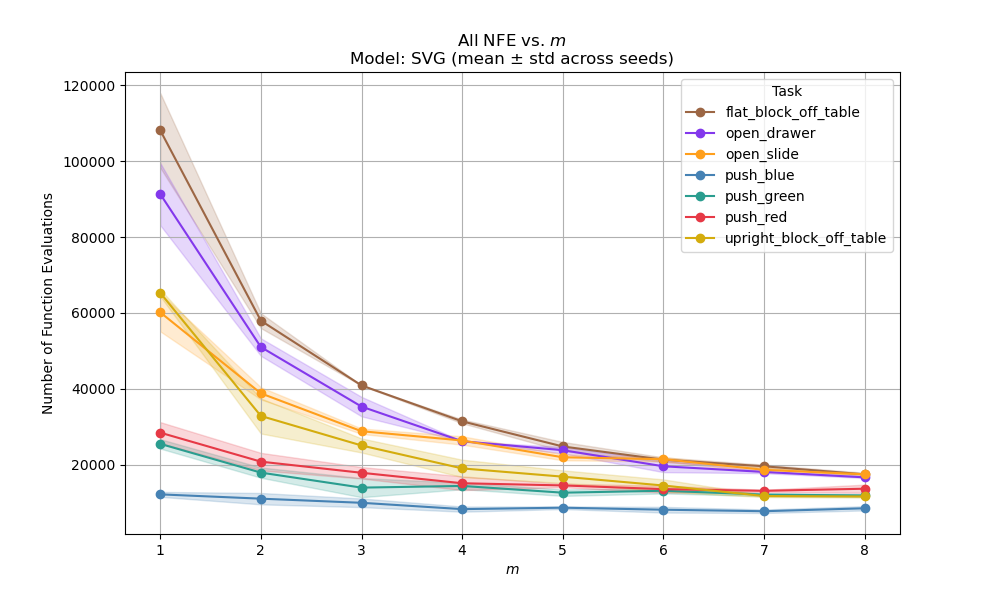}
    \end{minipage}
    \caption{VP2 fixed-step sweeps for the SVG world model. Left: task-wise success versus the reuse interval $m$. Right: the corresponding NFE versus $m$.}
    \label{fig:supp_step_svg}
\end{figure}

\begin{figure}[t]
    \centering
    \begin{minipage}[t]{0.49\linewidth}
        \centering
        \includegraphics[width=\textwidth]{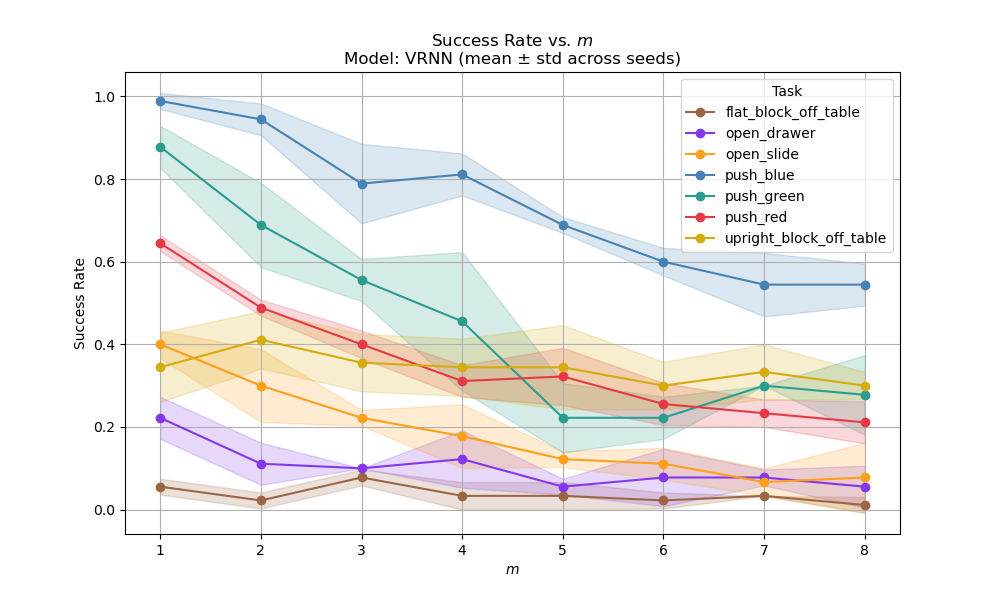}
    \end{minipage}\hfill
    \begin{minipage}[t]{0.49\linewidth}
        \centering
        \includegraphics[width=\textwidth]{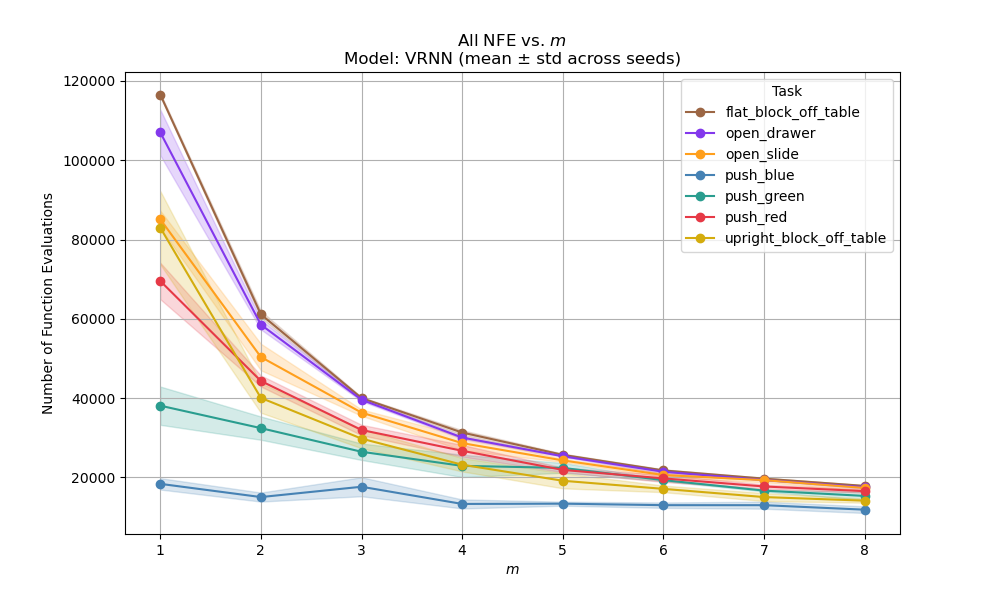}
    \end{minipage}
    \caption{VP2 fixed-step sweeps for the Struct-VRNN world model. Left: task-wise success versus $m$. Right: the corresponding NFE versus $m$.}
    \label{fig:supp_step_vrnn}
\end{figure}

Taken together, \cref{fig:supp_threshold_svg,fig:supp_threshold_vrnn,fig:supp_step_svg,fig:supp_step_vrnn} demonstrate that no single fixed $m$ or fixed $\epsilon$ is uniformly preferred across tasks and backbones. Some tasks tolerate aggressive plan reuse with little performance loss, while others degrade sharply once the reuse interval or deviation tolerance becomes too large.

\FloatBarrier
\subsection{Trajectory Visualization of Real-World Experiments}
We next visualize representative physical trajectories before reporting the exact per-task real-world counts used to form the aggregate summary in the main text.
\begin{figure}[H]
    \centering
    \includegraphics[width=0.98\linewidth]{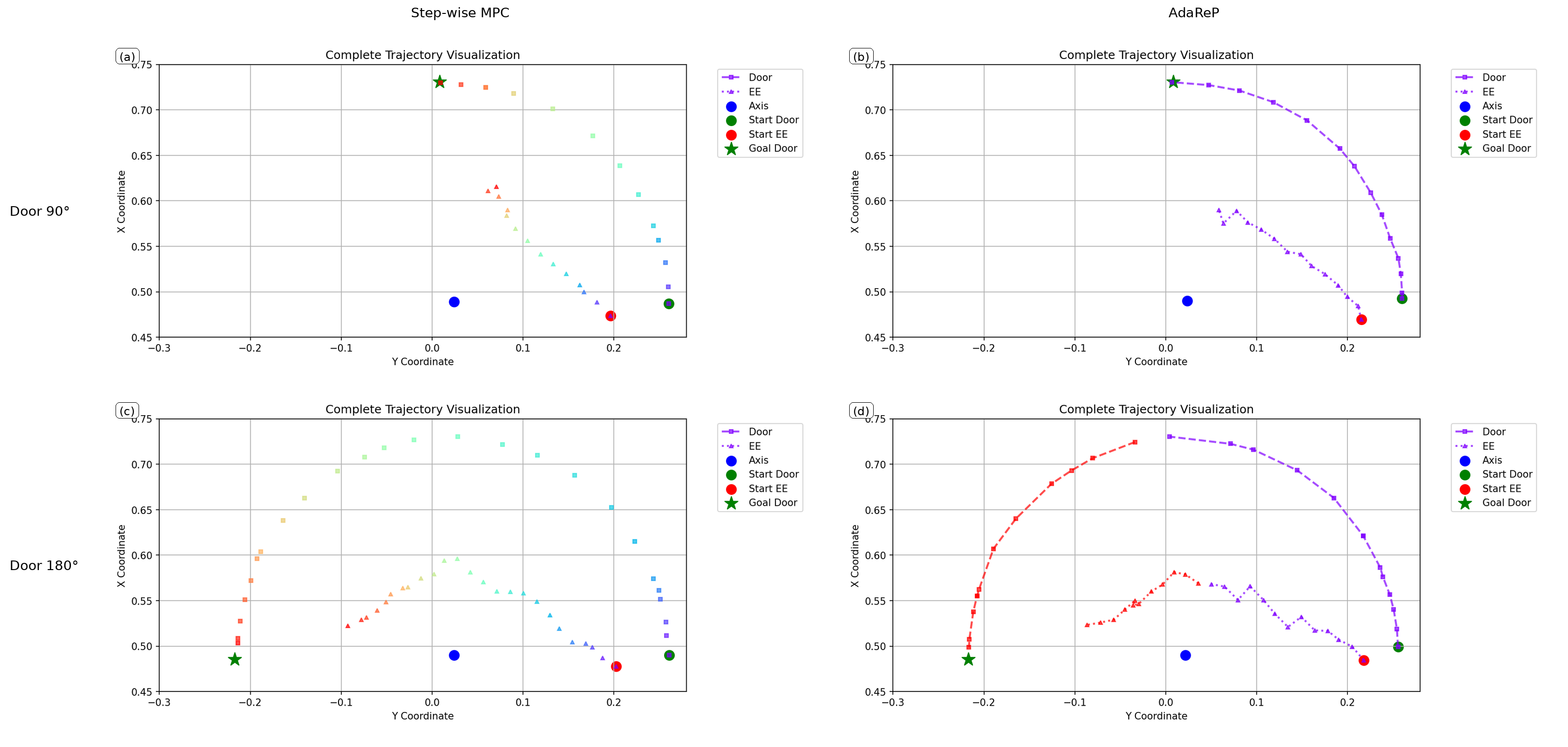}
    \caption{Real-robot door-opening trajectories. Rows correspond to door opening to $90^\circ$ and $180^\circ$, and columns compare step-wise MPC with AdaReP. Different colors indicate trajectory segments executed under different cached plans, so fewer color transitions visually indicate reduced replanning.}
    \label{fig:supp_franka_doors}
\end{figure}

The door tasks already exhibit the intended qualitative effect: the adaptive controller executes longer contiguous trajectory segments under the same cached plan. We subsequently visualize the longer-horizon T-block tasks, where the same reduction in plan switches is visible across translation, rotation, and the combined task.
\begin{figure}[H]
    \centering
    \includegraphics[width=0.90\linewidth]{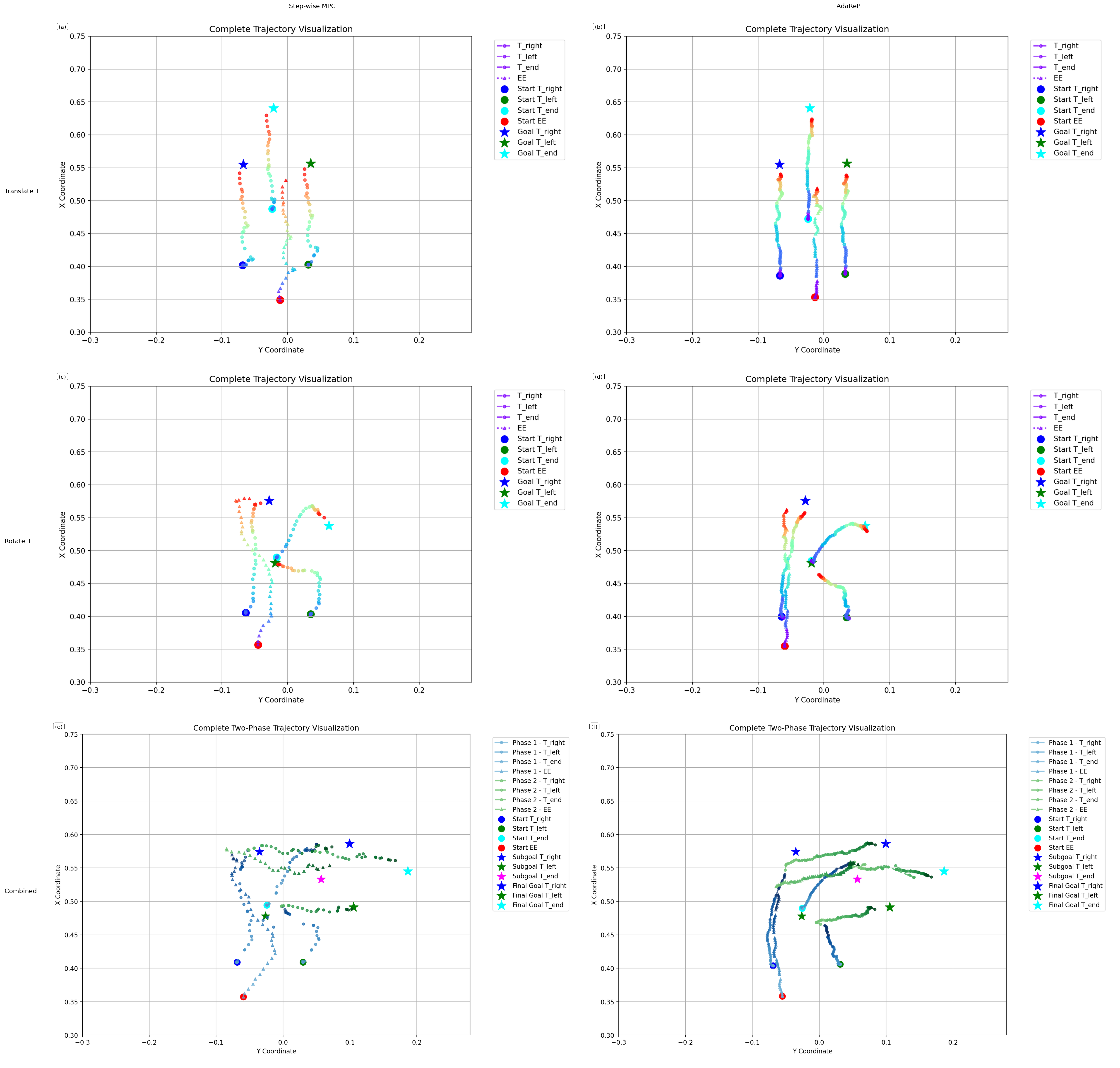}
    \caption{Real-robot T-block trajectories. Rows correspond to translation, rotation, and the combined task, and columns compare step-wise MPC with AdaReP. The qualitative reduction in plan switches mirrors the NFE savings reported quantitatively in the main text.}
    \label{fig:supp_franka_tblocks}
\end{figure}

\FloatBarrier
\subsection{Visual Demonstrations from Physical Rollouts}
In addition to the traced state-space trajectories above, we include representative frame sequences from the physical rollouts to show the robot--object interaction directly in image space (\cref{fig:supp_real_vis_doors,fig:supp_real_vis_tblocks}). These strips cover the same five real-world task instances summarized later in \cref{tab:supp_real_robot}: door opening to $90^\circ$ and $180^\circ$, T-block translation, T-block rotation, and the combined translation-and-rotation task.
\begin{figure}[H]
    \centering
    {\small\textbf{Top:} opening the door to $90^\circ$\\[1mm]}
    \includegraphics[width=0.98\linewidth]{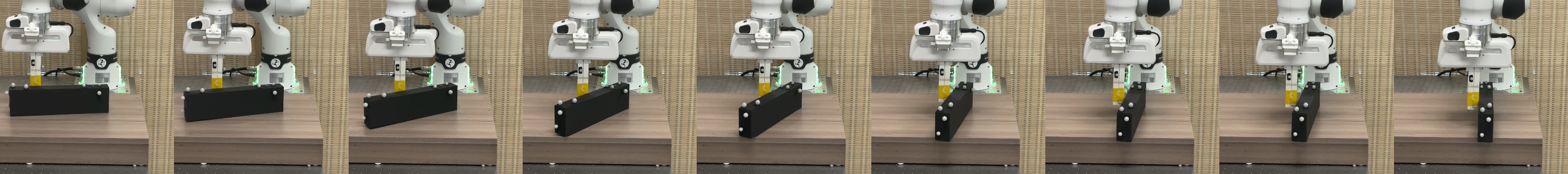}\\[1mm]
    {\small\textbf{Bottom:} opening the door to $180^\circ$\\[1mm]}
    \includegraphics[width=0.98\linewidth]{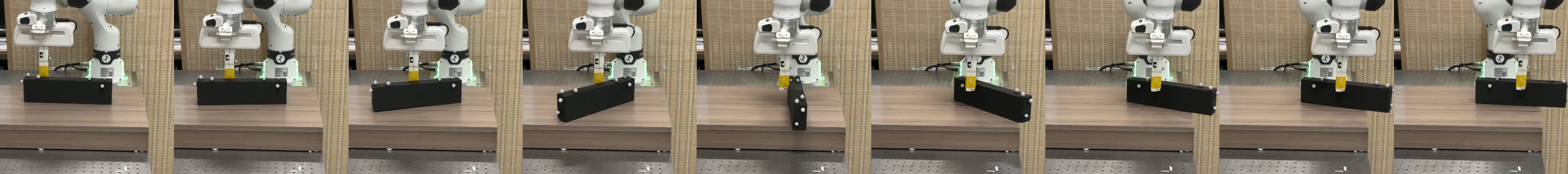}
    \caption{Representative frame sequences from the real-robot door-opening tasks. Each strip shows selected frames from a physical rollout and complements the state-space trajectory visualizations in \cref{fig:supp_franka_doors}.}
    \label{fig:supp_real_vis_doors}
\end{figure}

\begin{figure}[H]
    \centering
    {\small\textbf{Top:} translating the T-block\\[1mm]}
    \includegraphics[width=0.98\linewidth]{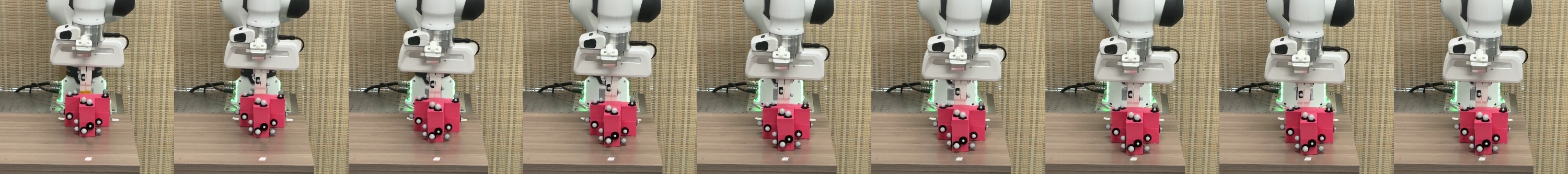}\\[1mm]
    {\small\textbf{Middle:} rotating the T-block\\[1mm]}
    \includegraphics[width=0.98\linewidth]{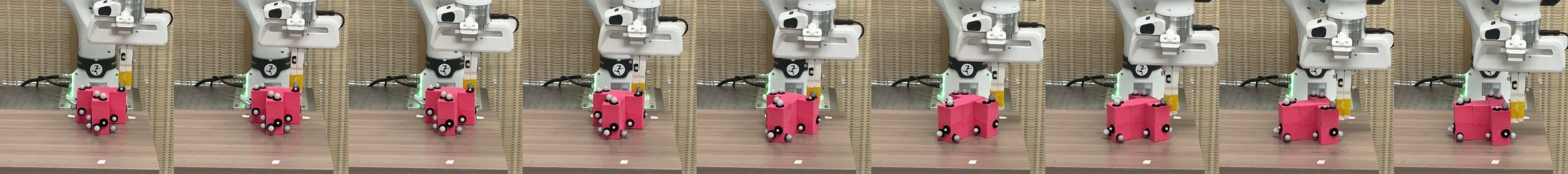}\\[1mm]
    {\small\textbf{Bottom:} translating and rotating the T-block\\[1mm]}
    \includegraphics[width=0.98\linewidth]{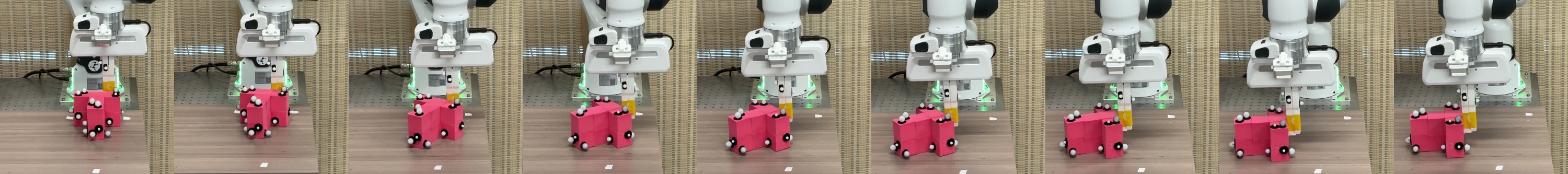}
    \caption{Representative frame sequences from the real-robot T-block tasks. These strips show the translation, rotation, and combined manipulation behaviors corresponding to the trajectory plots in \cref{fig:supp_franka_tblocks}.}
    \label{fig:supp_real_vis_tblocks}
\end{figure}

\FloatBarrier
\subsection{Per-task Real-World Results}
Table \ref{tab:supp_real_robot} records the exact task-instance counts and episode-level mean NFE values used to form the aggregate Franka summaries in the main text.
\begin{table}[H]
    \centering
    \caption{Exact task-instance outcomes for the Franka evaluation. Each task instance is evaluated for 10 trials. NFE values are the episode-level means used to compute the aggregate summaries in the main text.}
    \label{tab:supp_real_robot}
    \small
    \renewcommand{\arraystretch}{1.02}
    \begin{tabular}{>{\raggedright\arraybackslash}p{2.6cm}>{\centering\arraybackslash}p{1.45cm}>{\centering\arraybackslash}p{1.45cm}>{\centering\arraybackslash}p{1.55cm}>{\centering\arraybackslash}p{1.55cm}}
        \toprule
        Task instance & $\MPC_k^1$ success & \method success & $\MPC_k^1$ NFE & \method NFE \\
        \midrule
        Open door $90^\circ$ & $9/10$ & $9/10$ & $159\mathrm{K}$ & $14\mathrm{K}$ \\
        Open door $180^\circ$ & $6/10$ & $7/10$ & $198\mathrm{K}$ & $24\mathrm{K}$ \\
        Translate T & $8/10$ & $8/10$ & $345\mathrm{K}$ & $60\mathrm{K}$ \\
        Rotate T & $6/10$ & $7/10$ & $407\mathrm{K}$ & $64\mathrm{K}$ \\
        Combination & $5/10$ & $5/10$ & $834\mathrm{K}$ & $115\mathrm{K}$ \\
        \bottomrule
    \end{tabular}
\end{table}

\smallskip
\noindent\textit{Aggregate summary.} Aggregated over the 50 physical trials, the step-wise baseline achieves $34$ successes out of $50$ trials (success rate $0.68$) and AdaReP achieves $36$ successes out of $50$ trials (success rate $0.72$); the corresponding Wilson 95\% intervals are approximately $[0.54,0.79]$ and $[0.58,0.83]$. As noted in the main text, these counts should be read as descriptive evidence of a stable computation--performance trade-off rather than as a high-powered significance test.

\FloatBarrier

\end{document}